\documentclass{article}

\usepackage{hyperref}
\usepackage{footnote}
\usepackage{amssymb}
\usepackage{amsmath,amsthm}
\usepackage{stmaryrd}
\SetSymbolFont{stmry}{bold}{U}{stmry}{m}{n} 
\usepackage[noend]{algorithmic}
\usepackage{algorithm,caption}
\usepackage{multicol, paracol, multirow}
\usepackage{color}
\usepackage{float}
\usepackage{todonotes}
\usepackage{bm}
\usepackage[latin1]{inputenc}

\usepackage{thmtools}
\usepackage{thm-restate}

\usepackage{hyperref}
\usepackage{cleveref}

\usepackage{dsfont}

\usepackage{letltxmacro}

\newlength{\commentindent}
\makeatletter
\renewcommand{\algorithmiccomment}[1]{\unskip\hfill\makebox[\commentindent][l]{//~#1}\par}
\LetLtxMacro{\oldalgorithmic}{\algorithmic}
\renewcommand{\algorithmic}[1][0]{%
	\oldalgorithmic[#1]%
	\renewcommand{\ALC@com}[1]{%
		\ifnum\pdfstrcmp{##1}{default}=0\else\algorithmiccomment{##1}\fi}%
}
\makeatother

\def\colorful{0}

\usepackage{amsfonts,epsfig,color,graphicx,verbatim, enumitem}
\usepackage{multirow}

\newif\ifhyper\IfFileExists{hyperref.sty}{\hypertrue}{\hyperfalse}
\hypertrue
\ifhyper\usepackage{hyperref}\fi

\usepackage{enumitem}

\usepackage{framed}
\usepackage{nicefrac}

\def\nnewcolor{1}
\ifnum\nnewcolor=1

\fi
\ifnum\nnewcolor=0

\fi

\ifnum\colorful=1
\newcommand{\new}[1]{{\color{red} #1}}

\else
\newcommand{\new}[1]{{#1}}

\fi

\newtheorem{theorem}{Theorem}[section]

\newtheorem{lemma}[theorem]{Lemma}
\newtheorem{informal theorem}[theorem]{Theorem (informal statement)}

\newtheorem{corollary}[theorem]{Corollary}

\theoremstyle{definition}
\newtheorem{definition}[theorem]{Definition}
\newcommand{\eqdef}{\stackrel{{\mathrm {\scriptstyle def}}}{=}}

\newcommand{\R}{\mathbb{R}}

\newcommand{\Hyp}{\ensuremath{H}}
\newcommand{\eps}{\epsilon}

\newcommand{\pr}{\mathop{\mathbf{Pr}}}
\renewcommand{\Pr}{\mathbf{Pr}}
\newcommand{\poly}{\mathrm{poly}}

\newcommand{\sgn}{\mathrm{sign}}
\newcommand{\sign}{\mathrm{sign}}

\newcommand{\opt}{\mathrm{OPT}}

\DeclareMathOperator*{\E}{\mathbb{E}}

\newcommand{\C}{\mathcal{C}}

\newcommand{\MD}{\ensuremath{\mathtt{Mas}\{D_x, f, \eta(x)\}}}
\newcommand{\data}{\ensuremath{\mathcal{X}}}

\newcommand{\exor}{\mathrm{EX}}
\newcommand{\exorsim}{\mathrm{EXSim}}
\newcommand{\dest}{\ensuremath{\widehat{d}}}
\mathchardef\mhyphen="2D
\newcommand{\supp}{\ensuremath{\text{supp}}}
\newcommand{\lerr}{\ensuremath{\mathrm{err}^D_{0\mhyphen1}}} 
\newcommand{\ferr}{\ensuremath{\mathrm{err}^{D_x, f}_{0\mhyphen1}}} 
\newcommand{\adv}{\ensuremath{\mathrm{adv}}}
\newcommand{\Xrisky}{\ensuremath{\data^r}}
\newcommand{\Xsafe}{\ensuremath{\data^s}}

\newcommand{\X}{\ensuremath{\mathcal{X}}}
\newcommand{\Y}{\ensuremath{\mathcal{Y}}}
\newcommand{\negl}{\mathrm{negl}}
\newcommand{\boost}{\ensuremath{\mathtt{Massart\mhyphen Boost}}}
\newcommand{\wkl}{\ensuremath{\mathtt{WkL}}}
\newcommand{\wklbox}{\ensuremath{\mathtt{WkL_{box}}}}

\newcommand{\samp}{\ensuremath{\mathtt{Samp}}}
\newcommand{\estdens}{\ensuremath{\mathtt{Est\mhyphen Density}}}
\newcommand{\esterr}{\ensuremath{\mathtt{Est\mhyphen Error}}}
\newcommand{\overconf}{\ensuremath{\mathtt{OverConfident}}}

\newcommand{\wklr}{\ensuremath{\mathtt{rWkL}}} 
\newcommand{\examplegenerator}{\mathtt{EG}} 
\newcommand{\egr}{\ensuremath{\mathtt{rEG}}} 
\newcommand{\egh}{\ensuremath{\mathtt{hEG}}} 
\newcommand{\Xnoise}{\ensuremath{\data^{\mathrm{noisy}}}}

\newcommand{\Xset}{\ensuremath{\data^{\mathrm{H}}}} 

\newcommand{\bbboost}{\ensuremath{\mathtt{BlackBoxBoost}}} 
\newcommand{\B}{\ensuremath{\mathtt{SP}}} 
\newcommand{\pmone}{\ensuremath{\{\pm1\}}}
\newcommand{\F}{\ensuremath{\mathcal{F}}}
\newcommand{\wklcompiletime}{CT} 
\newcommand{\hypbitcomplexity}{BC_h} 
\newcommand{\hypevaltime}{R_h} 

\title{Boosting in the Presence of Massart Noise}

\author{
	Ilias Diakonikolas\thanks{Supported by NSF Award CCF-1652862 (CAREER), a Sloan Research Fellowship, and a DARPA Learning with Less Labels (LwLL) grant.} \\
	 University of Wisconsin-Madison \\ 
	 \texttt{ilias@cs.wisc.edu}  
	\and 
	Russell Impagliazzo\thanks{Supported by the Simons Foundation and NSF grant CCF-1909634.} \\ 
	 University of California-San Diego \\ 
	 \texttt{russell@eng.ucsd.edu}
	\and 
	Daniel Kane\thanks{Supported by NSF Award CCF-1553288 (CAREER) and a Sloan Research Fellowship.} \\ 
	 University of California-San Diego \\ 
	 \texttt{dakane@ucsd.edu}
	\and 
	Rex Lei\footnotemark[2] \\ 
	 University of California-San Diego \\ 
	 \texttt{rlei@eng.ucsd.edu}
	\and 
	Jessica Sorrell\footnotemark[2] \\ 
	 University of California-San Diego \\ 
	 \texttt{jlsorrel@eng.ucsd.edu}
	\and
	Christos Tzamos\thanks{Supported by the NSF grant CCF-2008006.} \\ 
	 University of Wisconsin-Madison \\ 
	 \texttt{tzamos@wisc.edu}
}

\begin{document}

\maketitle

\begin{abstract}
We study the problem of boosting the accuracy of a weak learner in the (distribution-independent) PAC model with Massart noise. In the Massart noise model, the label of each example $x$ 
is independently misclassified with probability $\eta(x) \leq \eta$, where $\eta<1/2$. The Massart
model lies between the random classification noise model and the agnostic model. 
Our main positive result is the first computationally efficient boosting algorithm 
in the presence of Massart noise that achieves misclassification error arbitrarily close 
to $\eta$. Prior to our work, no non-trivial booster was known in this setting. 
Moreover, we show that this error upper bound is best possible for polynomial-time 
black-box boosters, under standard cryptographic assumptions. Our upper and lower bounds
characterize the complexity of boosting in the distribution-independent PAC model with Massart noise. As a simple application of our positive result, we give the first efficient Massart learner
for unions of high-dimensional rectangles.
\end{abstract}

\setcounter{page}{0}

\thispagestyle{empty}

\newpage

\section{Introduction} \label{sec:intro}

\subsection{Background and Motivation} \label{ssec:background}

Boosting is a general learning technique that combines the outputs of a weak base learner
--- a learning algorithm with low but non-trivial accuracy --- to obtain a hypothesis of higher accuracy.
Boosting was introduced by Schapire~\cite{Schapire:90} and has since 
been extensively studied in machine learning and statistics. 
The reader is referred to~\cite{Schapire:03} for an early survey from the theoretical machine learning community,~\cite{BT07} for a statistics perspective, and~\cite{SF-book} for a book on the topic.
Here we study boosting in the context of learning classes of Boolean functions 
with a focus on Valiant's distribution-independent PAC model~\cite{val84}. 
During the past three decades, several efficient boosting procedures have been developed 
in the {\em realizable} PAC model, i.e., when the data is consistent with a function in the target class. 
On the other hand,  boosting in the presence of {\em noisy data} remains less understood.

In this work, we study the complexity of boosting in the presence of {\em Massart noise}. 
In the Massart (or bounded noise) model, the label of each example $x$ is flipped independently 
with probability $\eta(x) \leq \eta$, for some parameter $\eta<1/2$.
The flipping probability $\eta(x)$ is bounded but is unknown to the learner and 
can depend on the example $x$ in a potentially adversarial manner.
Formally, we have the following definition.

\begin{definition}[PAC Learning with Massart Noise] \label{def:massart-learning}
Let $\mathcal{C}$ be a concept class over $X= \R^n$, $D_{x}$
be any fixed but unknown distribution over $X$, and $0 \leq \eta <1/2$ be the noise parameter.
Let $f \in \mathcal{C}$ be the unknown target concept.
A {\em noisy example oracle}, $\mathrm{EX}^{\mathrm{Mas}}(f, D_{x}, \eta)$,
works as follows: Each time $\mathrm{EX}^{\mathrm{Mas}}(f, D_{x}, \eta)$ is invoked,
it returns a labeled example $(x, y)$, where $x \sim D_{x}$, $y = f(x)$ with
probability $1-\eta(x)$ and $y = -f(x)$ with probability $\eta(x)$,
for an {\em unknown} function $\eta(x): X \to [0, \eta]$.
Let $D$ denote the joint distribution on $(x, y)$ generated by the above oracle.
A learning algorithm is given i.i.d.\ samples from $D$
and its goal is to output a hypothesis $h$ such that with high probability
the misclassification error $\pr_{(x, y) \sim D} [h(x) \neq y]$ is as small as possible. We will use
$\opt \eqdef \inf_{g \in \mathcal{C}} \pr_{(x, y) \sim D} [g(x) \neq y]$ to denote the optimal misclassification error.
\end{definition}
\paragraph{Background on Massart Noise.} 
The Massart model is a natural semi-random input model that
is more realistic and robust than  random classification noise. Noise
can reflect computational difficulty or ambiguity, as well as random factors.
For example, a cursive  ``e'' might be substantially more likely to be 
misclassified as ``a'' than an upper case Roman letter.  Massart noise allows
for these variations in misclassification rates, 
while not requiring precise knowledge of which instances are more likely to be misclassified.
That is, algorithms that learn in the presence of Massart noise are likely to be less
brittle than those that depend on uniformity of misclassification noise.
Agnostic learning is of course even more robust, but unfortunately, it can be computationally infeasible to design agnostic learners for many applications.

In its above form, the Massart noise model was defined in~\cite{Massart2006}.
An essentially equivalent noise model had been defined in the 80s by Sloan and 
Rivest~\cite{Sloan88, Sloan92, RivestSloan:94, Sloan96}, 
and a very similar definition had been considered even earlier by Vapnik~\cite{Vapnik82}. 
The Massart model is a generalization of the Random Classification Noise (RCN) model~\cite{AL88}
and appears to be easier than the agnostic model~\cite{Haussler:92, KSS:94}.
Perhaps surprisingly, until very recently, no progress had been made on the efficient, distribution-free PAC 
learnability in the presence of Massart noise for any non-trivial concept class.

In more detail, the existence of an efficient distribution-independent PAC 
learning algorithm with non-trivial error guarantee for any concept class
in the Massart model had been posed as an open question
in a number of works, including~\cite{Sloan88, Cohen:97}, and was highlighted in 
A. Blum's FOCS'03 tutorial~\cite{Blum03}. Recent work~\cite{DGT19} made the first algorithmic progress 
in this model for the concept class of halfspaces. Specifically, \cite{DGT19} gave a polynomial-time learning 
algorithm for Massart halfspaces with misclassification error $\eta+\eps$. 
We note that the information-theoretically optimal error is 
$\opt= \E_{x \sim D_x} [\eta(x)]$, which is at most $\eta$ but could be much smaller. 
Thus, the error achieved by the aforementioned algorithm can be very far from optimal. 
Very recent follow-up work~\cite{CKMY20} showed that obtaining the optimal error 
of $\opt+\eps$ for halfspaces requires super-polynomial time in Kearns' Statistical Query (SQ) model~\cite{Kearns:98}.
Contemporaneous to the results of the current paper,~\cite{DK20-SQ-Massart} showed 
an SQ lower bound ruling out any constant factor or even polynomial factor approximation for this problem.
The approximability of learning Massart halfspaces remains a challenging open problem of current investigation.

\paragraph{Comparison to RCN and Agnostic Noise.}
Random Classification Noise (RCN)~\cite{AL88} is the special case of Massart
noise where the label of each example is independently flipped with probability {\em exactly} $\eta <1/2$.
RCN is a fundamentally easier model algorithmically.
Roughly speaking, RCN is predictable which allows us to cancel out the 
effect of the noise on any computation, in expectation. A formalization of this intuition is that
{\em any} Statistical Query (SQ) algorithm~\cite{Kearns:98} is automatically robust to RCN. 
This fact inherently fails in the presence of Massart noise. 
Roughly speaking, the ability of the Massart adversary to choose {\em whether} to flip a label
and if so, with what probability, makes this model algorithmically challenging. 
Moreover, the uniform noise assumption in the RCN model is commonly accepted to be unrealistic,
since in practical scenarios some instances are harder to classify than others~\cite{FV13}. 
For example, in the setting of human annotation noise~\cite{beigman2009learning}, it has been observed 
that the flipping probabilities are not uniform.

The agnostic model~\cite{Haussler:92, KSS:94} is the most challenging noise model in the literature, 
in which an adversary can arbitrarily flip an $\opt<1/2$ fraction of the labels. 
It is well-known that (even weak) learning in this model is computationally intractable 
for simple concept classes, including halfspaces~\cite{Daniely16}. 

The Massart model can be viewed as a reasonable compromise between RCN and the agnostic model, 
in the sense that it is a realistic noise model that may allow for efficient algorithms in settings 
where agnostic learning is computationally hard. This holds in particular for the important concept class 
of halfspaces. As already mentioned, even weak learning of halfspaces is hard in the agnostic model~\cite{Daniely16}, 
while an efficient Massart learner with non-trivial accuracy is known~\cite{DGT19}.

\paragraph{Boosting With Noisy Data.}
An important research direction, which was asked in Schapire's original 
paper~\cite{Schapire:90}, is to design boosting algorithms in the presence of noisy data.
This broad question has been studied in the past two decades by several researchers.
See Section~\ref{ssec:related} for a detailed summary of related work.
Specifically, prior work has obtained efficient boosters for RCN~\cite{KalaiServedio:03} 
and agnostic noise~\cite{Servedio:03jmlrshort, Feldman:10ab}. It should be emphasized that 
these prior works do not immediately extend to give boosters for the Massart noise setting. 
For example, while the agnostic model is stronger than the Massart model, 
an agnostic booster does not imply a Massart booster, as it relies on a much stronger assumption ---
the existence of a weak {\em agnostic} learner. That is, the complexity of noisy boosting is not ``monotone''
in the difficulty of the underlying noise model. More broadly, it turns out that
the complexity of boosting with inconsistent data, and the underlying boosting algorithms, 
crucially depend on the choice of the noise model. 

In this work, we ask the following question:

\begin{center}
{\em Can we develop efficient boosting algorithms for PAC learning with Massart noise?}
\end{center}

Our focus is on the distribution-independent setting. Given a distribution-independent Massart weak
learner for a concept class $\mathcal{C}$, we want to design a distribution-independent Massart learner 
for $\mathcal{C}$ with high(er) accuracy.
Prior to this work, no progress had been made on this front.
{\em In this paper, we resolve the complexity of the aforementioned problem by providing
(1) an efficient boosting algorithm and (2) a matching computational 
lower bound on the error rate of any black-box booster.}

This work is the first step of the broader agenda of  developing a general algorithmic 
theory of boosting for other ``benign'' semi-random noise models, lying between random and fully adversarial corruptions.

\subsection{Our Results} \label{ssec:results}
Our main result is the first computationally efficient boosting algorithm for (distribution-independent) 
PAC learning in the presence of Massart noise that guarantees misclassification arbitrarily close to $\eta$,
where $\eta$ is the upper bound on the Massart noise rate.
To state our main result, we will require the definition of a Massart 
weak learner (see Definition~\ref{def:massart-weak-learner} for additional detail).

\begin{definition}[Massart Weak Learner]\label{def:massart-wl-informal}
Let $\alpha, \gamma \in (0,1/2)$. An $(\alpha, \gamma)$-Massart weak learner $\wkl$ for concept class 
$\C$ is an algorithm that, for any distribution $D_x$ over examples, any function $f \in \C$,
and any noise function $\eta(x)$ with noise bound $\eta < 1/2-\alpha$,
outputs a hypothesis $h$ that with high probability satisfies $\Pr_{(x, y) \sim D}[h(x) \neq y] \leq 1/2-\gamma$,
where $D$ is the joint Massart noise distribution.
\end{definition}

We prove two versions of our main algorithmic result. In Section~\ref{sec:algorithm}, 
we give a somewhat simpler argument for the existence of a Massart noise-tolerant booster 
that converges within $O(1/(\eta\gamma^2))$ rounds of boosting (Theorem~\ref{thm:boosting}). 
In Section~\ref{sec:optimizations}, we give a more careful analysis of convergence, 
showing that the same algorithm in fact converges in $O(\log^2(1/\eta)/\gamma^2)$ rounds (Theorem~\ref{thm:opt-boosting}). 
In fact, the latter upper bound is nearly optimal for distribution-independent boosters 
(see, e.g., Chapter~13 of \cite{SF-book}). 
We now state our main result:

\begin{theorem}[Main Result] \label{thm:main-alg-informal}
There exists an algorithm $\boost$ that for every concept class $\mathcal{C}$, 
given samples to a Massart noise oracle $\mathrm{EX}^{\mathrm{Mas}}(f, D_{x}, \eta)$, where $f \in \mathcal{C}$, 
and black-box access to an $(\alpha, \gamma)$-Massart weak learner $\wkl$ for $\mathcal{C}$,  
$\boost$ efficiently computes a hypothesis $h$ that with high probability satisfies 
$\pr_{(x, y) \sim D} [h(x) \neq y] \leq \eta (1+O(\alpha))$. Specifically, $\boost$ makes
$O(\log^2(1/\eta)/\gamma^2)$ calls to $\wkl$ and draws 
$$\mathrm{polylog}(1/(\eta \gamma))/(\eta \gamma^2) \; m_{\wkl} + \poly(1/\alpha, 1/\gamma, 1/\eta)$$
samples from $\mathrm{EX}^{\mathrm{Mas}}(f, D_{x}, \eta)$, where $m_{\wkl}$ 
is the number of samples required by $\wkl$.
\end{theorem}

Prior to this work, no such boosting algorithm was known for PAC learning with Massart noise.
Moreover, as we explain in Section~\ref{ssec:related}, previous noise-tolerant boosters do not extend to the 
Massart noise setting. In Section~\ref{ssec:techniques}, we provide a detailed overview of our new algorithmic 
ideas to achieve this.

Some additional comments are in order. 
First, we note that the $\eta+\eps$ error guarantee achieved by our efficient booster can be far from
the information-theoretic minimum of $\opt+\eps$. The error guarantee of our generic 
booster matches the error guarantee of the best known polynomial-time learning algorithm 
for Massart halfspaces~\cite{DGT19}. Interestingly, the learning algorithm of~\cite{DGT19} can be viewed 
as a specialized boosting algorithm for the class of halfspaces. Theorem~\ref{thm:main-alg-informal} 
is a broad generalization of the latter result that applies to {\em any} concept class. This connection
was one of the initial motivations for this work.

A natural question is whether the error upper bound achieved by our booster 
can be improved. Perhaps surprisingly, we show that our guarantee is 
best possible for black-box boosting algorithms (under cryptographic assumptions).
Specifically, we have the following theorem:

\begin{theorem}[Lower Bound on Error of Black-Box Massart Boosting] \label{thm:lb-informal}
Assuming one-way functions exist, the following holds: No polynomial-time 
boosting algorithm, given black-box access to an $(\alpha, \gamma)$-Massart weak learner,
can output a hypothesis $h$ with misclassification error $\Pr_{(x, y) \sim D}[h(x) \neq y] \leq \eta (1+o(\alpha))$,
where $\eta$ is the upper bound on the Massart noise rate. In particular, this statement 
remains true on Massart distributions with optimal misclassification error $\opt \ll \eta$.
\end{theorem}

The reader is referred to Theorem~\ref{thm:lower-bound-detailed} for a detailed formal statement. 
Our lower bound establishes that the error upper bound achieved by our boosting algorithm 
is best possible.
It is worth pointing out a related lower bound shown in~\cite{KalaiServedio:03} in the context of 
RCN. Specifically,~\cite{KalaiServedio:03} showed that any efficient black-box booster tolerant to RCN
must incur error at least $\eta$ (with respect to the target function $f$), where $\eta$ is the RCN noise rate.
Since RCN is the special case of Massart noise where $\eta(x) = \eta$ for all $x$, the lower bound 
of~\cite{KalaiServedio:03} implies a lower bound of $\opt$ for black-box Massart boosting.
Importantly, our lower bound is significantly stronger, as it shows a lower bound of $\eta$, 
even when $\opt$ is much smaller than $\eta$. 

\new{Intriguingly, Theorem~\ref{thm:lb-informal} shows that the error guarantee of the~\cite{DGT19} 
learning algorithm for Massart halfspaces cannot be improved using boosting, 
and ties with recent work~\cite{DK20-SQ-Massart} providing evidence 
that learning with Massart noise (within error relative to $\opt$) is computationally hard.}

\paragraph{Application: Massart Learning of Unions of Rectangles.}

As an application of Theorem~\ref{thm:main-alg-informal}, we give the first efficient learning algorithm
for unions of (axis-aligned) rectangles in the presence of Massart noise. 
Interestingly, weak agnostic learning of a single rectangle is computationally hard in the agnostic model (see, e.g.,~\cite{FGRW09}).
Recall that a rectangle $R \in \mathbb{R}^d$ is an intersection of inequalities of the form 
$x \cdot v < t$,  where $v \in \{\pm e_j : j \in [d]\}$ and $t \in \mathbb{R}$.
Formally, we show:

\begin{theorem} \label{thm:rect-inf}
There exists an efficient algorithm that learns unions of $k$ rectangles on $\R^d$ with Massart noise bounded by $\eta$. 
The algorithm has sample complexity $k d^{O(k)} \mathrm{poly}(1/\eps, 1/\eta)$, 
runs in time $(kd^k/\eps)^{O(k)} \mathrm{poly}(1/\eta)$, 
and achieves misclassification error $\eta + \epsilon$, for any $\epsilon > 0$. 
\end{theorem}

See Theorem~\ref{thm:rect-formal} for a more detailed statement.
Theorem~\ref{thm:rect-inf} follows by an application of Theorem~\ref{thm:main-alg-informal} coupled with a 
simple weak learner for unions of rectangles that we develop. Our weak learner finds a rectangle 
entirely contained in the negative region to gain some advantage over a random guess.

It is worth pointing out that the Massart SQ lower bound of~\cite{CKMY20} applies to learning monotone 
conjunctions. This rules out efficient SQ algorithms with error $\opt+\eps$, even for a single rectangle.

\subsection{Overview of Techniques} \label{ssec:techniques}
In this section, we provide a brief overview of our approach.

\paragraph{Boosting Algorithm Approach.} 
We start with our Massart boosting algorithm.
Let $D$ be the Massart distribution $\MD$ from which our examples are drawn.
The distribution $D_x$ on examples is fixed but arbitrary and the function $\eta(x)$ is a Massart noise function satisfying $\eta(x) \leq \eta <1/2$ with respect to the target function $f \in \mathcal{C}$. 
As is standard in distribution-independent boosting, our boosting algorithm adaptively generates a sequence of distributions 
$D^{(i)}$, invokes the weak learner on samples from these distributions,
and incrementally combines the corresponding weak hypotheses to obtain a hypothesis with higher accuracy. 

The technical challenge of distribution-independent boosting is the adaptive generation of new distributions $D^{(i)}$ that effectively use the weak learner to acquire new and useful information about the target function $f$. 
To see why this requires some care, consider an adversarial weak learner that attempts to give the booster as little information about $f$ as possible, while still satisfying its definition as a weak learner. 
Such an adversarial weak learner might, whenever possible, produce hypotheses that correctly classify the same, small set of examples $P$, while classifying all other examples randomly. Assuming the function $f$ is balanced, and the intermediate distributions $D^{(i)}$ assign probability at least $\gamma$ to $P$, these adversarial hypotheses will have accuracy $1/2 + \gamma$ on their corresponding distributions, while providing no new information about the target function to the booster. To thwart this behavior, the booster must eventually restrict its distributions to assign sufficiently small probability to $P$ to ensure that the weak learner can no longer meet its promised accuracy lower-bound by correctly classifying only the set $P$. In this way, the booster can force the weak learner to output hypotheses correlating with $f$ on other subsets of its domain. Under reasonable conditions on the specific strategy for reweighting distributions, boosters that incrementally decrease the probability assigned to examples as they are more frequently correctly classified by weak hypotheses are known to eventually converge to high-accuracy hypotheses, by reduction to iterated two-player zero-sum games \cite{FS97}. This general approach to reweighting intermediate distributions is common to all distribution-independent boosters, even in the noiseless setting.

Our booster follows the smooth boosting framework~\cite{Servedio:03jmlrshort} 
with some crucial modifications that are necessary to handle Massart noise. 
A smooth boosting algorithm generates intermediate distributions that do not put too much weight on any individual point, and so do not compel the weak learner to generate hypotheses having good correlation only with noisy examples. 
This makes the smooth boosting framework a natural starting point for the design of a Massart noise-tolerant booster, though smoothness of the intermediate distributions alone is not a sufficient condition for preservation of the Massart noise property.

To see why, note that to preserve the Massart noise property of the intermediate distributions, it is not enough to enforce an upper bound on the probability that any (potentially noisy) example can be assigned. 
We require an upper bound on the \emph{relative} probabilities of sampling noisy and correct labels for a given point, 
to ensure we always have a noise upper bound $\eta^{(i)} < 1/2$. This seems to suggest that preserving the Massart noise property requires a corresponding lower bound on the probability assigned to any given example, so that we do not inadvertently assign more probability to $(x, -f(x))$ than $(x, f(x))$. This is immediately at odds with our strategy for making 
use of an adversarial weak learner, since guaranteeing progress requires that our distributions can assign arbitrarily small probability 
to some examples. So, we must use alternative techniques to manage noise.\footnote{We note that vanilla smooth boosting 
has been shown to succeed in the agnostic model. Interestingly, the above subtle issue for Massart boosting does not arise in agnostic boosting, since agnostic noise is easy to preserve.}.

The fix for this is to simply not include examples $(x,y)$ in the support of $D^{(i)}$ whenever including them could violate the Massart noise property or permit an adversarial weak learner to tell us only what we already know. If many of the weak hypotheses obtained by our booster agree with the label $y$ on $x$, then we learn little from a marginal weak hypothesis that agrees with $y$ on $x$, and so we exclude $(x,y)$ from the support of $D^{(i)}$. 
We must also symmetrically exclude $(x, -y)$, otherwise we risk violating the Massart noise property for $D^{(i)}$, since we have assigned no probability to $(x,y)$, and it may be the case that $-y \neq f(x)$. Withholding these examples allows the booster to get new information from the weak learner at each round, 
without ever invoking it on an excessively noisy sample. 

This balance comes at the cost of updates from the weak learner on withheld examples. 
This may not seem to pose a significant problem for our booster at first. After all, points on which many hypotheses agree are points 
where our algorithm is already fairly confident about the correct value of $f(x)$. Unfortunately, this
confidence may not be sufficiently justified to ensure an $\eta+\eps$ error at the end of the day. 
In order to deal with this, our algorithm will need to make use of one further idea. 
We directly check the empirical error of our aggregated hypotheses on the set of withheld examples. 
If this error is too large (i.e., larger than $\eta+\eps$), we conclude we have ``more to learn" 
about the withheld examples after all. Since even an adversarial weak learner will give us new information 
about these examples in expectation, we include them in subsequent distributions, 
with appropriate upper and lower bounds on their probabilities to preserve the Massart noise property. 
If the empirical error is not too large, we are content to learn nothing new about these examples, 
and so continue to withhold them for the next round of boosting.

Overall, our algorithm will alternate between the two steps of
applying the weak learner to an appropriately reweighted version of
the underlying distribution, and checking the consistency of our hypotheses 
with the set of withheld examples. Each step will allow us to make progress 
in the sense of decreasing a relevant potential function. We iterate these steps until almost all points are consistently being withheld from the weak learner. Once we reach this condition, we will have produced a hypothesis with appropriately small error, 
and can terminate the algorithm.
We analyze the convergence of our algorithm to a low-error hypothesis via a novel potential function 
that can be easily adapted to analyze other smooth boosting algorithms. 

\paragraph{Error Lower Bound.} 
We show that no ``black-box'' generic boosting algorithm for Massart noise can have significantly better error 
than that for our algorithm, i.e., $\eta + \Theta(\eta \alpha)$. While this seemingly matches the lower bound 
for RCN boosting from~\cite{KalaiServedio:03}, the RCN bound only implies a lower bound for the special case of Massart noise 
when $\eta = \opt$. We show this lower bound extends unchanged for a small but polynomial value of $\opt$.
That is, boosting algorithms cannot be improved even when only a very small fraction of instances are actually noisy.  
To prove our lower bound, we consider a situation where the function to be learned is highly biased, 
and there is a small fraction of inputs with the majority value that are noisy and indistinguishable from non-noisy inputs.   
If the distribution queried by a boosting algorithm does not reweigh values in some way to favor the minority answer, 
the weak learner can return the majority answer and have high correlation.  On the other hand, if it does reweigh values, 
it risks adding too much noise to the small fraction of already noisy examples, violating the Massart condition. 
The standard methods for constructing meaningful weak learner queries involve downweighing examples $(x,y)$ with the majority label, specifically by (1) rejecting them via rejection sampling and (2) keeping $x$ and using additional noise to perturb the label $y$.
Although we allow the boosting algorithm to produce distributions arbitrarily, we show that, essentially, 
these are the only computationally feasible options, and that each has the aforementioned limitations.

\subsection{Comparison with Prior Work} \label{ssec:related}

The literature on boosting is fairly extensive. Since the introduction of the technique
by Schapire~\cite{Schapire:90}, boosting has become one of the most studied areas
in machine learning --- encompassing both theory and practice. 
Early boosting algorithms~\cite{Schapire:90, Fre95, FreundSchapire:97} 
were not tolerant in the presence of noisy data.  In this section, we summarize
the most relevant prior work with a focus on boosting techniques that have provable
noise tolerance guarantees.

Efficient boosting algorithms have been developed for PAC learning in the agnostic model~\cite{Haussler:92, KSS:94}
and in the presence of Random Classification Noise (RCN)~\cite{AL88}.
The notion of agnostic boosting was introduced in~\cite{BLM:01}. Subsequently, 
a line of work~\cite{Servedio:03jmlrshort, Gavinsky:03short, KMV:08, KalaiK09, Feldman:10ab} 
developed efficient agnostic boosters with improved error guarantees, culminating in the optimal bound.
These agnostic boosters rely on one of two techniques: smooth boosting, introduced in~\cite{Servedio:03jmlrshort},
or boosting via branching programs, developed in~\cite{MansourMcAllester:02}. 
While both of these techniques have been successful in the agnostic model,
the only known booster tolerant to RCN is due to~\cite{KalaiServedio:03}, and relies
on the branching programs technique~\cite{MansourMcAllester:02}.
In the following paragraphs, we briefly summarize these two techniques.

Smooth boosting~\cite{Servedio:03jmlrshort} is a technique that produces intermediate distributions
which do not assign too much weight on any single example. The technique was inspired
by Impagliazzo's hard-core set constructions in complexity theory~\cite{Impagliazzo:95}
(see also~\cite{KS99, Holenstein05, BarakHK09}) and is closely related to convex optimization. 
Roughly speaking, smooth boosting algorithms are reminiscent of first-order methods 
in convex optimization. Smooth boosting methods have been
shown to be tolerant to agnostic noise~\cite{Servedio:03jmlrshort, Gavinsky:03short, KalaiK09, Feldman:10ab}.
Interestingly,~\cite{LongS10} established a lower bound against potential-based 
convex boosting techniques in the presence of RCN. While we do not prove any relevant theorems
here, we believe that our technique can be adapted to give an efficient booster in the presence of RCN.

Another important boosting technique relies on branching programs~\cite{MansourMcAllester:02}.
The main idea is to iteratively construct a branching program in which each internal node
is labeled with a hypothesis generated by some call to the weak learner. This technique
is quite general and has led to noise tolerant boosters for both (RCN)~\cite{KalaiServedio:03} 
(see also~\cite{LongServedio:05, LongServedio:08nips} for refined and simplified boosters relying on this technique)
and agnostic noise~\cite{KMV:08}. Roughly speaking, the branching programs 
methodology leads to ``non-convex algorithms'' and is quite flexible.  

It is worth pointing out that the aforementioned branching program-based boosters 
do not succeed with Massart noise in their current form. Specifically, the RCN booster in~\cite{KalaiServedio:03} 
crucially relies on the uniform noise property of RCN, which implies that agreement with the true target function 
is proportional to agreement with the observed labels. On the other hand, for the agnostic booster of~\cite{KMV:08}, 
the generated distributions on which the weak learner is invoked do not preserve the Massart noise property --- a crucial
requirement for any such booster. While it should be possible to adapt the branching program technique to work
in the Massart noise model, we believe that the smooth-boosting technique developed in this paper 
leads to simpler and significantly more efficient boosters that are potentially practical.

Finally, we acknowledge existing work developing efficient learning algorithms for Massart halfspaces 
(and related noise models) 
in the {\em distribution-specific} PAC model~\cite{AwasthiBHU15, AwasthiBHZ16, ZhangLC17, DKTZ20, Zhang20, DKTZ20b, DKKTZ20}. These works are technically orthogonal to the results of this paper,
as they crucially leverage a priori structural information about the distribution on examples (e.g., log-concavity).

\subsection{Organization} \label{ssec:org}
The structure of this paper is as follows: Section~\ref{sec:prelims} contains preliminary definitions and fixes notation. In Section~\ref{sec:algorithm} we present our Massart noise-tolerant boosting algorithm. In Section~\ref{sec:optimizations}, we prove an improved round complexity for our booster. In Section~\ref{sec:lower-bound}, we show that the error achieved by our booster is optimal by proving a lower-bound on the error of any black-box  Massart noise-tolerant booster. In Section~\ref{sec:applications}, we give an application of our boosting algorithm to learning unions of rectangles.

\section{Preliminaries}\label{sec:prelims}

Throughout this work, we use the notation $ S \;||\; z$ to denote appending $z$ to a sequence $S$. For a distribution $D$ over domain $\X$, we write $\supp(D)$ to denote the set of all $x \in \X$ such that $D(x) \neq 0$. For a function $f$ mapping its domain $\X$ to $\R$, we define $\sgn(f): \X \rightarrow \{\pm 1\}$ to be the function
\[\sgn(f)(x) := \begin{cases} 1 & \text{ if } f(x)\geq 0 \\ -1 & \text{ if } f(x) < 0 \end{cases}\]

\subsection{Massart Noise Model}

Let $\C$ be a class of Boolean-valued functions over some domain $\data$, 
and let $D_x$ be a distribution over $\data$. 
Let $f \in \C$ be an unknown target function, 
and let $\eta(x): \data \rightarrow [0,1/2)$ be an unknown function. 

\begin{definition}[Noisy Example Oracle]
When invoked, noisy example oracle $\exor^{\mathrm{Mas}}(f, D_x, \eta(x))$ produces a labeled example $(x,y)$ as follows:
$\exor^{\mathrm{Mas}}(f, D_x, \eta(x))$ draws $x \sim D_x$. 
With probability $\eta(x)$, it returns $(x, -f(x))$, and otherwise returns $(x, f(x))$.
\end{definition}

The noisy example oracle induces a Massart distribution. 
\begin{definition}[Massart Distribution]
A Massart distribution $D := \MD$ over $(\data, \pm1)$ is the distribution induced by sampling from $\exor^{\mathrm{Mas}}(f, D_x, \eta(x))$. 
\end{definition}
We refer to $\eta(x)$ in this context as the \emph{Massart noise function}.

We say a Massart distribution $D$ has \emph{noise rate} $\eta$ if 
$\eta(x) \le \eta$ for all $x \in \supp(D_x)$. 
The \emph{noise bound} of a Massart noise function is $\eta$ if
$\max_{x \in \supp(D_x)}\eta(x) = \eta$. 
We emphasize that this model restricts the noise bound to be $\eta < 1/2$.

\subsection{Learning under Massart Noise}

Let $f: \data \rightarrow \{\pm1\}$ be a function in concept class $\C$.
Let $D = \MD$ be a Massart noise distribution over $\data$. 

\begin{definition}[Misclassification Error]
The misclassification error of hypothesis $h : \data \rightarrow \{\pm 1\}$ over $D$ is
$$ \lerr(h) = \Pr_{(x,y) \sim D} [h(x) \ne y]$$
\end{definition}

\begin{definition}[Function Error]
The error of hypothesis $h : \data \rightarrow \{\pm 1\}$ with respect to $f$ over $D$ is
$$ \ferr(h) = \Pr_{x \sim D_x} [h(x) \ne f(x)] \;.$$
\end{definition}

\begin{definition}[Advantage]
Hypothesis $h: \data \rightarrow \{\pm 1\}$ has advantage $\gamma > 0$ against distribution $D$ if $\lerr(h) \le 1/2 - \gamma$. 
Equivalently, $h$ has advantage $\gamma$ against distribution $D$ if
$$\tfrac12 \E_{(x,y) \sim D} [h(x) \cdot y] \ge \gamma \;.$$
\end{definition}

We use the notation $\adv^D(h)$ to denote the largest $\gamma \in [0,1/2]$ for which $\lerr(h) \leq 1/2 - \gamma$.

\subsection{Boosting and Weak Learners}

\begin{restatable}[Massart Noise Weak Learner]{definition}{weaklearnerdef}\label{def:massart-weak-learner}
Let $\C$ be a concept class of functions $f: \data \rightarrow \{\pm1\}$.
Let $\alpha \in [0,1/2)$. Let $\gamma: \R \rightarrow \R$ be a function of $\alpha$.
A Massart noise $(\alpha, \gamma)$-weak learner $\wkl$ for $\C$ is an algorithm such that, 
for any distribution $D_x$ over $\data$, function $f \in \C$,
and noise function $\eta(x)$ with noise bound $\eta < 1/2-\alpha$,
$\wkl$ outputs a hypothesis $h: \data \rightarrow \{\pm 1\}$ such that
$$\Pr_{S}[\adv^D(h) \ge \gamma] \geq 2/3 \;,$$
where the sample $S$ is drawn from Massart noise distribution $D = \MD$.
\end{restatable}

Parametrizing Massart noise weak learners by $\alpha$ allows for more precise analysises of Massart boosting. 
Massart noise weak learners may be able to provide better guarantees when given input distributions with lower noise rates. 
For designers of Massart weak learners seeking to apply our boosting algorithm, this parametrization may inform comparisons among multiple weak learners for the same problem. 
For reference, our unions of rectangles weak learner (Section~\ref{sec:applications}) is an $(\alpha, \alpha^2/O(d)^k)$-Massart noise weak learner, where $d$ and $k$ parametrize the concept class. 
However, $\gamma$ being polynomially related to $\alpha$ is not a strict requirement for applying our boosting algorithm.

We also observe that the probability of the weak learner returning a hypothesis with advantage less than $\gamma$ can be driven down to any target failure probability $\delta$, through standard repetition arguments. 
\begin{lemma}[$\wkl$ repetition]\label{lem:wkl-rep}
	Let $\wkl$ be an $(\alpha, \gamma)$-Massart noise weak learner requiring a sample of size $m_{\wkl}$. Then for any $\delta \in (0, 1/3)$, $2\log(2/\delta)$ calls to $\wkl$ and $2\log(2/\delta)(m_{\wkl} + 1/\gamma^2)$ examples suffice to obtain a hypothesis with advantage at least $\gamma/2$ with all but probability $\delta$. 
\end{lemma}

\begin{proof}
	To drive down the failure probability of $\wkl$, we draw $2\log(2/\delta)$ samples of size $m_{\wkl}$ and run $\wkl$ on each of them to obtain a list of hypotheses, at least one of which has advantage $\gamma$ with all but probability $\delta/2$. We then draw a sample of size $2\log(2/\delta)/\gamma^2$ to test each hypothesis in our list, keeping the best. The Chernoff-Hoeffding inequality guarantees that testing our hypotheses overestimates the advantage by more than $\gamma/2$ with probability no greater than $\delta/2$, and so we obtain a hypothesis with advantage at least $\gamma/2$ with all but probability $\delta$. 
\end{proof}

Algorithmically, our boosting algorithm creates weak learner queries by reweighing the input Massart distribution $D$; this process may increase the noise rate. By knowing the weak learner's maximum noise tolerance $1/2 - \alpha$, our boosting algorithm avoids reweighing $D$ too much. (See Section~\ref{sec:algorithm}) for more details.)
In Section~\ref{sec:lower-bound}, we prove a slightly stronger than $\eta$ lower bound on the error of black-box boosting algorithms under Massart noise ---  parametrizing by $\alpha$ allows us to quantify how much stronger.

Note that we define an $(\alpha, \gamma)$-Massart noise weak learner to have failure probability at most $1/3$. 
We observe that for any desired $\delta \in (0, 1/3)$, such a weak learner can be used to obtain a hypothesis with advantage $\gamma/2$, with all but probability $\delta$, by standard repetition techniques.

We are primarily interested in {\it efficient} Massart noise weak learners (Definition~\ref{def:efficient-massart-weak-learner}).

\begin{restatable}[Efficient Massart Noise Weak Learner]{definition}{efficientweaklearnerdef}
	\label{def:efficient-massart-weak-learner}
	Let $\wkl$ be an $(\alpha, \gamma)$-Massart noise weak learner. Let $n$ be the maximum bit complexity of a single example $(x,y) \in \data \times \pmone$, and let $m_\wkl$ denote the number of examples comprising sample $S$. 
	$\wkl(S)$ is \emph{efficient} if
	\begin{enumerate}
		\item  $\wkl$ uses $m_\wkl(n, \eta, \gamma) = \poly(n, 1/(1-2\eta), 1/\gamma)$ examples.  
		\item $\wkl$ outputs a hypothesis $h$ in time $\poly(n,  1/(1-2\eta), 1/\gamma)$.
		\item Hypothesis $h(x)$ has bit complexity $\poly(n, 1/(1-2\eta), 1/\gamma)$.
		\item For all $x \in \data$, the hypothesis $h(x)$ can be evaluated in time $\poly(n, 1/(1-2\eta), 1/\gamma)$. 
	\end{enumerate}
\end{restatable} 

Boosting algorithms utilize the advantage guarantee of the weak learner by cleverly reweighting its input distributions. 
To sample from these reweighted distributions, we sample from the underlying distribution $D$ via $\exor^{\mathrm{Mas}}(f, D_x, \eta(x))$ and reject examples according to
a function $\mu: (\data, \{\pm 1\}) \rightarrow [0,1]$. We refer to $\mu$ informally as a \emph{measure} to emphasize that it induces a distribution, but need not be one itself. 

\begin{definition}[Rejection Sampled Distribution $D_\mu$]\label{def:rej-sample}
	Let $D$ be a Massart noise distribution, and let $\mu: \data \times \{\pm 1\} \rightarrow [0, 1]$ be an efficiently computable measure.
	We define $D_\mu$ as the distribution generated from $D$ by the following rejection sampling procedure: 
	draw an example $(x,y) \sim D$. 
	With probability $\mu(x,y)$, keep this example. Otherwise, repeat this process (until an example is kept).
\end{definition}

Note that some choices of $\mu$ may induce a distribution $D_\mu$ which is \emph{not} Massart, as reweighting examples may distort $\eta(x)$, and so it is possible that we no longer have a noise bound less than $1/2$.
In particular, if there is an $x \in \supp(D_x)$ for which $\mu(x, -f(x))\eta(x) \gg \mu(x, f(x))(1-\eta(x))$,
then $D_\mu$ is not a Massart noise distribution and running the weak learner on a sample from this distribution is not guaranteed to return a hypothesis with good advantage. In designing our boosting algorithm, we will choose $\mu$ carefully to ensure that this never happens.

The expectation of the measure $\mu$ with respect to the underlying distribution $D$ is a useful quantity for analyzing distribution-independent boosting algorithms. It will affect the sample complexity of making calls to the weak learner and, looking ahead, will be used to bound the error of the final hypothesis output by our algorithm.
\begin{definition}[Density of a measure]\label{def:dens}
	Let $D$ be a Massart noise distribution, and let $\mu$ be a measure. The \emph{density} of $\mu$ with respect to $D$ is
	\[d(\mu) \eqdef \E_{(x,y)\sim D}[\mu(x,y)]. \]
\end{definition}

\begin{lemma}[Sampling from $D_\mu$]\label{lem:rejsamp}
	For any $m > 0$, $\delta \in (0, 1/2)$, obtaining a sample of size $m$ from $D_{\mu}$ by rejection sampling from $D$ requires no more than 
	$$\frac{\log(1/\delta)}{d(\mu)^2} + \frac{2m}{d(\mu)}$$ examples from distribution $D$, with all but probability $\delta$.
\end{lemma}

\begin{proof}
 From the definition of $D_{\mu}$, we can sample from $D_{\mu}$ by drawing an example $(x,y)$ from $D$ and keeping it with probability $\mu(x,y)$. By Definition~\ref{def:dens}, we expect to keep an example with probability $d(\mu)$. Then the Chernoff-Hoeffding inequality allows us to conclude that, following this procedure, if we draw $\frac{\log(1/\delta)}{d(\mu)^2} + \frac{2m}{d(\mu)}$ examples from $D$, we keep at least $m$ of them with all but probability $\delta$.
\end{proof}

\section{Boosting Algorithm}\label{sec:algorithm}

In this section, we present our Massart noise-tolerant boosting algorithm $\boost$ (Algorithm~\ref{alg:boost}) and prove the following theorem:
\begin{theorem}[(Simplified) Boosting Theorem]\label{thm:simp-boosting}
	Let $\wkl$ be an $(\alpha, \gamma)$-weak learner requiring a sample of size $m_{\wkl}$.
	Then for any Massart distribution $D$ with noise rate $\eta < 1/2$, and any $\epsilon > \frac{8\eta\alpha}{1 - 2\alpha}$, $\boost$ will
	\begin{itemize}
		\item make at most $T \in \tilde{O}\left(1/(\eta\gamma^2)\right)$ calls to $\wkl$
		\item output a hypothesis $\Hyp$ such that $\lerr(\Hyp) \leq \eta + \epsilon$ and $\ferr(\Hyp) \leq \frac{\eta + \epsilon}{1 - \eta}$
		\item make $$m \in \tilde{O}\left(\frac{1}{\eta\gamma^2\epsilon^3} +  \frac{m_{\wkl}}{\eta^2\gamma^2 } + \frac{1}{\eta^2\gamma^4} \right)$$ calls to its example oracle $\exor^{\mathrm{Mas}}(f, D_x, \eta(x))$
		\item run in time $$\tilde{O}\left( \frac{m_{\wkl}}{\eta^3\gamma^4} + \frac{1}{\eta^3\gamma^6} + \frac{1}{\eta^2\gamma^4\epsilon^3} \right),$$ neglecting the runtime of the weak learner. 
	\end{itemize}	
\end{theorem}

We define parameters and relevant properties of the algorithm in Section~\ref{ssec:alg-defs}, along with a high level pseudocode sketch (Algorithm ~\ref{alg:boost-sketch}). We describe our algorithm and its subroutines in more detail, and provide a more complete pseudocode description, in Section~\ref{ssec:alg-desc}. In Section~\ref{sec:convergence}, we analyze the round complexity of our booster. We give upper-bounds on the final error of our booster in Section~\ref{sec:err-bounds}. In Section~\ref{sec:samp-complexity}, we show upper-bounds on sample complexity. We prove our main boosting theorem in Section~\ref{sec:boost-theorem}. Finally, in Section~\ref{ssec:hypothesis}, we comment on the form of our final hypothesis. We note that improved round and sample complexity can be shown by a more careful analysis of convergence, but we defer the analysis to Section~\ref{sec:optimizations} and Theorem~\ref{thm:opt-boosting}.

\subsection{Overview of Boosting Algorithm}
As in most distribution-independent boosting algorithms, to leverage our weak learner to construct a high-accuracy learner, we iteratively simulate new distributions for the weak learner. We use each weak hypothesis to incrementally improve a working hypothesis, where our working hypothesis is $\sgn(G)$ for some efficiently evaluable function of the form $G: \X \rightarrow \R$, initialized to the constant function 0. Up to a point, we update the distribution over examples by increasing the weight on examples misclassified by $\sgn(G)$ and decreasing the weight on examples on which $\sgn(G)$ is correct. 
We simulate these new distributions using our example oracle $\exor^{\mathrm{Mas}}(f, D_x, \eta(x))$ by rejection sampling according to an appropriately selected measure $\mu: \data \rightarrow [0,1]$, and then run our weak learner on the resulting sample. 

However, rejection sampling according to an arbitrary $\mu$ may destroy the Massart noise property by producing a Massart noise distribution with noise rate greater than $1/2$.  
For example, if $\mu$ sets $\mu(x, f(x)) = 0$ and $\mu(x, -f(x)) = 1$ for some $x \in \data$, then the rejection sampled distribution $D_\mu$ has an effective noise rate of 1.  
In these cases, the behavior of the weak learner on any sample drawn from such a reweighted distribution would be undefined by Definition~\ref{def:massart-weak-learner}.
 
To ensure we maintain a noise bound below $1/2$ for all intermediate distributions $D_{\mu}$, we must define $\mu$ so that for all $x \in \X$, it does not induce a distribution from which $(x,f(x))$ is less likely to be sampled than $(x, -f(x))$. Of course, neither the noise $\eta(x)$ nor the correct label $f(x)$ for a given $x$ are known to $\boost$, and so we set $\mu(x,y)=0$ for any example $(x,y)$ at risk of violating this constraint. To effectively use the weak learner to improve $G$, $\mu(x,y)$ will be negatively correlated with $yG(x)$, which implies an example will be at risk when $|G(x)|$ is large. We use $G$, then, to partition $\X$ into two sets: $\Xrisky$ and $\Xsafe$.

The set $\Xsafe$ contains all $x\in \X$ for which $|G(x)|$ is not too large, ensuring that the effective noise rate of each $x \in \Xsafe$ is bounded away from $1/2$. Thus, it is ``safe" to call the weak learner on examples $(x,y)$ where $x \in  \Xsafe$. Initially, all $x \in \data$ are in $\Xsafe$. 

The set $\Xrisky$ contains all $x \in \X$ for which $|G(x)|$ exceeds a specific threshold. We will have $\mu(x,y) = 0$ for all $x \in \Xrisky$, independent of label, and so these examples are removed from the support of $D_{\mu}$ as desired. At the same time, the weak learner's advantage is guaranteed with respect to the distribution from which its sample was drawn, and so we are not guaranteed any sort of improvement in expectation for these risky examples.  

To ensure that we end up with a low-error hypothesis, $\boost$ performs an additional calibration step. If the working hypothesis misclassifies too many risky examples, it must be ``overconfident" in its predictions on this set, and so we can improve $G$ with the hypothesis $-\sgn(G)$ (similar to the balancing step of \cite{Feldman:10ab}). This recalibration step moves all $x$ back to $\Xsafe$, allowing us to again call the weak learner on these examples.
As more examples are correctly classified by $\sgn(G)$, the density of the measure $\mu$ decreases. 
When this density is small, the algorithm terminates and returns the classifier $\sgn(G)$.

\subsection{Definitions}\label{ssec:alg-defs}

Our boosting algorithm makes use of a function $M: \R \rightarrow [0,1]$ to reweight examples drawn from $D$. We define this function

\[M(v) \eqdef \begin{cases}
1 & v < 0, \\
e^{- v} & 0 \leq v. 
\end{cases}\]

We can now define the measure function used by the boosting algorithm to reweight labeled examples at each round of boosting. This function is parameterized by $s \in \R_{> 0}$, and a real-valued function $F: \X \rightarrow \R$, and will assign no weight to examples $(x,y)$ such that $|F(x)| \geq s$.
\begin{definition}[Measure function]\label{def:measure-fn}
	We denote a measure over examples by $\mu(x,y)$ 
	and define
	\[\mu_{F, s}(x,y) \eqdef \begin{cases} M(yF(x)) & \text{if } |F(x)| < s, \\ 0 & \text{otherwise.}\end{cases}\]

\end{definition}

Looking ahead, the parameter $s$ will be set once and for all at the start of the algorithm. Its purpose is to allow our boosting algorithm to operate with $(\alpha, \gamma)$-weak learners for which $\alpha > 0$ (i.e., which require Massart noise rate $\eta$ bounded away from 1/2 by some positive quantity). To ensure that the noise rate for each example in the rejection sampled distribution $D_{\mu}$ (see Definition~\ref{def:rej-sample}) is never more than $1/2 - \alpha$, we let $c = \frac{4\eta\alpha}{1 - 2\alpha}$ and define 
$$s \eqdef \log\left(\frac{1 - \eta}{\eta + c}\right). $$
This parameter $s$ is exactly the threshold for $|G(x)|$ above which an example $(x,y)$ is considered risky rather than safe.

We will primarily refer to the measure function $\mu_{f, s}$ defined by taking $f$ to be $G_t$, the current state of the real-valued hypothesis at round $t$ of boosting. We will often refer to the measure $\mu$ at round $t$, so to simplify notation somewhat, we define
	\[\mu_{t}(x,y) \eqdef \begin{cases} M(yG_t(x)) & \text{if } |G(x)| < s, \\ 0 & \text{otherwise,}\end{cases}\]
where we have suppressed the $s$ subscript as well, since it is fixed throughout the algorithm.

At each round $t$ of boosting, we will partition the domain $\X$ into two sets: $\Xsafe_t$ and $\Xrisky_t$. If it is ``safe" to run the weak learner on a sample containing $x$, we say $x \in \Xsafe_t$. Otherwise, $x \in \Xrisky_t$. 

\begin{definition}[$\Xsafe_G$]
	For all $x \in \data$,
	$x \in \Xsafe_G$ if $|G(x)| < s$. 
\end{definition}

\begin{definition}[$\Xrisky_G$]
	For all $x \in \data$,
	$x \in \Xrisky_G$ if $|G(x)| \ge s$.
\end{definition}
To simplify notation, we denote $\Xsafe_{G_t}$ by $\Xsafe_t$ and similarly $\Xrisky_{G_t}$ by $\Xrisky_t$.

\begin{figure}[H]
	\begin{algorithm}[H]
		\caption{$\boost^{\exor^{\mathrm{Mas}}(f, D_x, \eta(x)), \wkl}(\eta, \epsilon, \gamma$) \\
			$\eta$: Massart noise rate \\
			$\epsilon$: Target error in excess of $\eta$ \\
			$\gamma$: Weak learner advantage guarantee
		}\label{alg:boost-sketch}
		\begin{algorithmic}
			\STATE $G \gets 0$,
			\STATE $\widehat{d} \gets 1$
			\WHILE{$d(\mu) > \eta$}
			\STATE $S \gets $ sample from $D_{\mu}$
			\STATE $h \gets  \wkl(S)$ 
			\STATE $h^{s}(x) \gets \begin{cases} h(x) & \text{ if }  x \in \Xsafe_{G}, \\ 0 & \text{ otherwise } \end{cases}$
			\STATE $G \gets G + \lambda h^{s}$
			\IF{ error of $\sgn(G)$ on $\Xrisky_G$ exceeds $\eta + \epsilon$ }
			\STATE $h^{r}(x)  \gets \begin{cases} - \sgn(G(x)) & \text{ if } x \in \Xrisky_G, \\ 0 & \text{ otherwise } \end{cases}$
			\STATE $G \gets G + \lambda h^{r}$
			\ENDIF
			\STATE Update $\mu$ according to Definition~\ref{def:measure-fn}
			\ENDWHILE
			\STATE $\Hyp \gets \sgn(G_{t})$
			\RETURN $\Hyp$
		\end{algorithmic}
	\end{algorithm}
\end{figure}

\subsection{Description of Boosting Algorithm}\label{ssec:alg-desc}

The full pseudocode for our boosting algorithm is given in Algorithm~\ref{alg:boost}. We begin by initializing the hypothesis $G_0$ to the constant 0 function. By definition of the measure function, this initialization of $G_0$ sets $\mu_0(x,y) = 1$ for all examples $(x,y)$ in the support of $D$, and so we initialize our estimate for the density of the measure $\widehat{d}$ to 1. The boosting algorithm then proceeds in rounds in which it first queries the weak learner on safe examples, updates the hypothesis $G$ with the new weak hypothesis, and checks the error of the updated hypothesis $G$ on risky examples. At the end of the round, the density of the current measure $\mu_G$ is estimated, and if it has fallen below the threshold $\kappa$, the algorithm terminates and outputs the final hypothesis $\sign(G)$. Within each round, the booster makes calls to three subroutines (in addition to the weak learner): $\samp$ (Routine~\ref{alg:samp}), $\estdens$ (Routine~\ref{alg:est-density}), and $\overconf$ (Routine~\ref{alg:est-error}), which we first describe informally. 

The $\samp$ subroutine encapsulates the process of drawing samples for the weak learner from the reweighted distributions constructed by the booster. The $\samp$ procedure is given oracle access to $\exor^{\mathrm{Mas}}(f, D_x, \eta(x))$, so that is can sample from $D$. The $\samp$ procedure takes as input a function (the current hypothesis) $G$, the size of the sample $m_{\wkl}$ required by the weak learner, and the threshold $s$ for $|G(x)|$ that defines which examples are to be withheld from the weak learner. $\samp$ repeatedly draws examples from $\exor^{\mathrm{Mas}}(f, D_x, \eta(x))$, and keeps them with probability $\mu_{G,s}(x,y)$, the value of which is computed using $G$ and $s$. After $\samp$ has drawn a sample of size $m_\wkl$, it returns the sample, and this is what is given to the weak learner as input.

The subroutine $\estdens$ is used to estimate the current density of the measure $\mu_t$, which is necessary to test the termination condition of our algorithm. Ideally, the algorithm terminates once $d(\mu_t) < \kappa$, and so $\estdens$ is called at the end of each round of boosting to estimate this density. $\estdens$ is given oracle access to $\exor^{\mathrm{Mas}}(f, D_x, \eta(x))$ and takes as input $G$ and $s$, so that it can empirically estimate $d(\mu)$ using a sample drawn from $D$. $\estdens$ also takes as input three parameters: $\delta_{\mathtt{dens}}$, $\epsilon$, and $\eta$. The parameters $\epsilon$ and $\eta$ are used to specify the desired accuracy of the density estimation, $\beta = \min\{\epsilon/2, \eta/4\}$. The parameter $\delta_{\mathtt{dens}}$ specifies the tolerable probability of failure of the density estimation procedure (i.e., the probability that $\estdens$ returns an estimate of $d(\mu)$ with error greater than $\beta$). 

The subroutine $\overconf$ determines when the error of $\sgn(G)$ on examples withheld from the weak learner (i.e., examples in $\Xrisky_G$) has grown too large. If, at round $t$, the probability mass on $\Xrisky_t$ is large, and the error of $\sgn(G_t)$ on $\Xrisky_t$ exceeds $\eta + \epsilon$, we must improve $G_t$ on these examples to reach our target error of $\eta + \epsilon$. Because we withhold examples in $\Xrisky_{t}$ from the weak learner at round $t+1$, we are not guaranteed that the next weak hypothesis, $h_{t+1}$, will provide any amount of progress in expectation on these examples, so some additional steps are needed to improve $G_t$. The role of $\overconf$ is to estimate whether the probability mass of $\Xrisky_t$ is significant, and if so, whether the error of $\sgn(G_t)$ on $\Xrisky_t$ is large enough that an additional correction step is necessary. 

The subroutine $\overconf$ is given oracle access to $\exor^{\mathrm{Mas}}(f, D_x, \eta(x))$ and takes as input $G$, $s$, $\eta$, and $\epsilon$. It also takes an additional parameter $\delta_{\mathtt{err}}$, which specifies the tolerable probability of failure for $\overconf$ (i.e., the probability that $\overconf$ returns a false positive or false negative). $\overconf$ first estimates the probability that $|G(x)| > s$. If it estimates $\Pr_{(x,y) \sim D}[|G(x)| \geq s] < \epsilon/4$, then the overall contribution of examples in $\Xrisky_G$ to the total error of $\sgn(G)$ is sufficiently small that the correction step is not needed. In this case, $\overconf$ returns false. If it estimates the probability to be greater than $\epsilon/4$, it draws a new sample for estimating the conditional error of $G$ on $\Xrisky_G$. The subroutine makes calls to $\exor^{\mathrm{Mas}}(f, D_x, \eta(x))$ and keeps only the examples $(x,y)$ such that $x\in \Xrisky_G$. It draws a sufficiently large sample to estimate the error of $G$ on this set to within $\epsilon/4$ with all but probability $\delta_{\mathtt{err}}/2$. If the estimated error exceeds $\eta + 3\epsilon/4$, it returns true and the correction step takes place. Otherwise it returns false, as the conditional error on these points is tolerable. 

\begin{figure}[H]
\begin{algorithm}[H]
	\caption{$\boost^{\exor^{\mathrm{Mas}}(f, D_x, \eta(x)), \wkl}(\lambda, \kappa, \eta, \epsilon, \delta, \gamma, \alpha,  m_{\wkl})$ \\
	$\lambda$: Learning rate \\
	$\kappa$: Target density for measure $\mu$ \\
	$\eta$: Massart noise rate \\
	$\epsilon$: Target error in excess of $\eta$ \\
	$\delta$: Target failure probability for $\boost$ \\
	$\gamma$: Weak learner advantage guarantee \\
	$\alpha$: Weak learner parameter indicating $\wkl$ can tolerate noise rate $\eta < 1/2 - \alpha$ \\
	$m_{\wkl}$: Sample size for $\wkl$
	}\label{alg:boost}
	\begin{algorithmic}
		\STATE $c \gets \frac{4\eta\alpha}{1 - 2\alpha}$, $s \gets \log\left(\frac{1 - \eta}{\eta + c}\right)$
		\COMMENT{set parameters for managing noise}
		\STATE $\delta_{\mathtt{err}} \gets \delta\eta\gamma^2/1536$, $\delta_{\mathtt{dens}} \gets \delta\eta\gamma^2/1024$
		\COMMENT{set failure probabilities for subroutines}
		\STATE $G_0 \gets 0$, $t \gets 0$
		\COMMENT{initialize $G$, round counter} 
		\STATE $\widehat{d} \gets 1$
		\COMMENT{initialize density estimate}
		\WHILE{$\dest > \kappa$}
			\STATE $t \gets t + 1$
			\STATE $S \gets \samp^{\exor^{\mathrm{Mas}}(f, D_x, \eta(x))}(G_{t-1}, n ,s)$ 
			\COMMENT{draw a sample for the weak learner}
			\STATE $h_t \gets  \wkl(S)$ 
			\COMMENT{obtain a weak hypothesis}
			\STATE $h^{s}_t(x) \gets \begin{cases} h_t(x) & \text{ if }  x \in \Xsafe_{t-1}, \\ 0 & \text{ otherwise } \end{cases}$
			\COMMENT{zero out hypothesis on $\Xrisky_t$}
			\STATE $G_t \gets G_{t-1} + \lambda h^{s}_t$
			\COMMENT{update working hypothesis}
			\IF{$\overconf^{\exor^{\mathrm{Mas}}(f, D_x, \eta(x))}(G_t, s, \delta_{\mathtt{err}}, \epsilon)$ }
			\STATE $h^{r}_t(x)  \gets \begin{cases} - \sgn(G_t(x)) & \text{ if } x \in \Xrisky_t, \\ 0 & \text{ otherwise } \end{cases}$
			\COMMENT{if error on $\Xrisky_t$ is high, be less confident}
			\STATE $G_t \gets G_t + \lambda h_t^{r}$
			\COMMENT{update working hypothesis}
			\ENDIF
			\STATE $\dest \gets \estdens^{\exor^{\mathrm{Mas}}(f, D_x, \eta(x))}(G_t, s, \delta_{\mathtt{dens}}, \epsilon) $
			\COMMENT{estimate density of measure}
		\ENDWHILE
		\STATE $\Hyp \gets \sgn(G_{t})$
		\RETURN $\Hyp$
	\end{algorithmic}
\end{algorithm}
\end{figure}
\begin{figure}
\begin{algorithm}[H]
	 \floatname{algorithm}{Routine}
	\caption{$\samp^{\exor^{\mathrm{Mas}}(f, D_x, \eta(x))}(G, m_{\wkl}, s)$}\label{alg:samp}
	\begin{algorithmic}
		\STATE $S \gets \emptyset$
		\WHILE{$|S| \leq m_{\wkl}$}
		\STATE $(x,y) \gets \exor^{\mathrm{Mas}}(f, D_x, \eta(x))$
		\STATE With prob $\mu_{G,s}(x,y)$, $S\gets S \;||\; (x,y) $ 
		\COMMENT{draw a sample for $\wkl$ from $D_{\mu}$}
		\ENDWHILE 
		\RETURN $S$
	\end{algorithmic}
\end{algorithm}
\begin{algorithm}[H]
	\floatname{algorithm}{Routine}
	\caption{$\estdens^{\exor^{\mathrm{Mas}}(f, D_x, \eta(x))}(G, s, \delta_{\mathtt{dens}}, \epsilon, \eta)$}\label{alg:est-density}
	\begin{algorithmic}
		\STATE $\beta \gets \min\{\epsilon/2, \eta/4\}$
		\STATE Draw set $S$ of $\log(1/\delta_{\mathtt{dens}})/(2\beta^2)$ examples from $\exor^{\mathrm{Mas}}(f, D_x, \eta(x))$
		\STATE $\dest \gets \frac{1}{|S|}\sum_{(x,y) \in S} \mu_{G,s}(x,y)$
		\COMMENT{estimate the density of $\mu$}
		\RETURN $\dest$
	\end{algorithmic}
\end{algorithm}
\begin{algorithm}[H]
	\floatname{algorithm}{Routine}
	\caption{$\overconf^{\exor^{\mathrm{Mas}}(f, D_x, \eta(x))}(G, s, \delta_{\mathtt{err}}, \epsilon, \eta)$}\label{alg:est-error}
	\begin{algorithmic}
		\STATE Draw set $S$ of $32\log(2/\delta_{\mathtt{err}})/\epsilon^2$ examples from $\exor^{\mathrm{Mas}}(f, D_x, \eta(x))$
		\IF{$\frac{|S \cap \Xrisky_G|}{|S|} \leq \epsilon/4$}
			\RETURN \FALSE
			\COMMENT{if $\Xrisky_G$ is small, return false }
		\ENDIF  
		\STATE $S \gets \emptyset$
		\WHILE{$|S| \leq 8\log(2/\delta_{\mathtt{err}})/\epsilon^2$}
			\STATE $(x,y) \gets \exor^{\mathrm{Mas}}(f, D_x, \eta(x))$ 
			\IF{$|G(x)| \geq s$}		
				\STATE $S \gets S \;||\; (x,y)$
			\ENDIF
		\ENDWHILE
		\STATE $\widehat{\epsilon} \gets \frac{1}{2|S|}\sum_{(x,y) \in S} |y - \sgn(G(x))|$
		\COMMENT{if $\Xrisky_G$ is large, estimate error on $\Xrisky_G$ }
		\IF{$\widehat{\epsilon} \geq \eta + 3\epsilon/4$}
		\RETURN \TRUE
		\COMMENT{if error and $\Xrisky_G$ are large, return true}
		\ELSE
		\RETURN \FALSE
		\COMMENT{if error is small, return false}
		\ENDIF
	\end{algorithmic}
\end{algorithm}
\end{figure}

\subsection{Convergence of $\boost$}\label{sec:convergence}

In this subsection, we bound the error of the final hypothesis output by Algorithm~\ref{alg:boost} and the number of rounds of boosting required to achieve this error bound. 
We begin by showing an invariant of our algorithm that will be useful in subsequent potential arguments. 
\begin{lemma}[Invariant for $|G_t(x)|$]\label{lem:invariant}
	For all rounds $t$ of boosting and all examples $(x,y) \in (\X, \Y)$, at the end of round $t$, $|G_t(x)| < s + \lambda$.
\end{lemma}

\begin{proof}
We first show that at the end of round $t$, $|G_t(x)| \leq s + \lambda$. On examples $x$ such that $|G_{t-1}(x)| \geq s$, either $G_t(x) = G_{t-1}(x) - \lambda \sgn(G_{t-1}(x))$ or $G_t(x) = G_{t-1}(x)$,
and so $|G_t(x)| \leq |G_{t-1}(x)|$. Since $|G_t(x)| \geq |G_{t-1}(x)|$ only when $|G_{t-1}(x)| < s$, we now consider how much larger it can be. For examples such that $|G_{t-1}(x)| < s$, either $G_t(x) = G_{t-1}(x) + \lambda h_t(x) + \lambda h^c_t(x)$ (when $\overconf$ returns true) or $G_t(x) = G_{t-1}(x) + \lambda h_t(x)$. In the first case, if $\lambda h^c_t(x) \neq 0$, then $\sgn(G_{t-1}(x) + h_t(x)) = -h^c_t(x)$, and so $|G_t(x)| \leq |G_{t-1}(x) + \lambda h_t(x)|$ for both cases. Since the hypothesis $h_t(x)$ output by the weak learner has codomain $[-1,1]$, it follows that $|G_t(x)| \leq  |G_{t-1}(x) + \lambda h_t(x)| < s + \lambda$. 
\end{proof}

We now bound from below the progress that Algorithm~\ref{alg:boost} makes in each round of boosting, according to the potential function introduced below. Intuitively, our choice of potential function is motivated by the observation that $\E_{(x,y) \sim D}[M(yG(x))]$ is a reasonable proxy for the misclassification error of the corresponding hypothesis $\Hyp(x) = \sgn(G(x))$ (recall that $M(yG(x)) = 1$ whenever $\sgn(G(x)) \neq y$ and $M(yG(x))$ goes to 0 with $yG(x)$). We might consider using the expectation $\E_{(x,y)\sim D}M(yG(x))$ as a potential function itself, but because $M(v)$ is constant for all $v \leq 0$, this candidate potential function fails to reflect progress on examples that the current hypothesis misclassifies. 
  
 We instead consider the function $\phi_t(x,y)$

\[\phi_t(x,y) = \int_{yG_t(x)}^{\infty} M(z) dz\]
and the potential function
\[\Phi(t) = \E_{(x,y) \sim D} [\phi_t(x,y)] = \E_{(x,y) \sim D} \int_{yG_t(x)}^{\infty} M(z) dz. \]
To see how this function allows us to capture the incremental progress made at each round, consider how the potential $\Phi(t)$ changes as we take a step of size $\lambda$ from $G_t$ in the direction of some hypothesis $h$. If we take $\lambda$ sufficiently small, then we have from the mean value theorem that the change in potential should be not too much smaller than $\E_{(x,y)\sim D} \lambda yh_t(x)M(yG_t(x))$. Supposing for a moment that the function $\mu_{G,s}(x,y)$ used for reweighting were \emph{exactly} $M(yG(x))$, then a hypothesis $h$ with advantage $\gamma$ would guarantee a change in potential of roughly 
\begin{align*}
\E_{(x,y)\sim D} \lambda yh_t(x)\mu_t(x,y) = \E_{(x,y)\sim D} \lambda yh_t(x)D_{\mu_t}(x,y)d(\mu_t) = \lambda \gamma d(\mu_t),
\end{align*}
and so would show we can leverage a weak learner to make progress with respect to this function at each round. Note that the change in potential above is proportional the current density of $\mu_t$.

Due to the constraints of our Massart noise weak learner, however, we cannot take $\mu_t(x,y) = M(yG_t(x))$ for all $(x,y)$. Instead, we permit the algorithm to make no progress, or even regress, on examples for which $\mu_t(x,y) \neq M(yG_t(x))$, but show that in expectation over all examples, progress is still made. We then use the relationship between $d(\mu_{G,s})$ and the error $\lerr(\sgn(G))$ to show that if we are no longer making progress against $\Phi$ round to round, $d(\mu_{G,s})$, and therefore $\lerr(\sgn(G))$ must be small. 

We will make use of the following upper-bound on $d(\mu_t)$ in terms of $\Phi(t)$. 
\begin{lemma}[Potential upper-bounds density]\label{lem:dens-pot}
	For every round $t$ of $\boost$,  $d(\mu_t) \leq \Phi(t)$.
\end{lemma}

\begin{proof}
	We show that $d(\mu_t) \leq \Phi_t$ by showing $\mu_t(x,y) \leq \phi_t(x,y)$. For examples $(x,y)$ such that $yG_t(x) > 0$, we have 
	\begin{align*}
	\phi_t(x,y) = \int_{yG_t(x)}^{\infty}e^{-z} dz = e^{-yG_t(x)} \geq \mu_t(x,y).
	\end{align*}
	
	For the remaining points, we simply observe that either $\mu_t(x,y) = 1$ or $\mu_t(x,y) = 0$. In either case, the potential  
	\begin{align*}
	\phi_t(x,y) = \int_{yG_t(x)}^{\infty}M(z) dz \geq \int_{0}^{\infty}e^{-z} dz = 1
	\end{align*}
	and so we have that $\mu_t(x,y) \leq \phi_t(x,y)$, and therefore $d(\mu_t) \leq \Phi_t$.
\end{proof}

We now prove what we have informally claimed above, that $\boost$ makes progress against $\Phi$ at each round. 
\begin{lemma}[Potential Drop]\label{lem:potential}
	Take $\lambda = \gamma/8$, $\delta_{\wkl} = \delta\eta\gamma^2/1536$, and assume  $\epsilon \geq \frac{8\eta\alpha}{1 - 2\alpha}$. Then for every round of boosting $t$, with all but probability $\delta\eta\gamma^2/768$,
	\[\Phi(t) - \Phi(t+1) \geq \frac{\gamma^2}{32}\left(d(\mu_t) - \frac{\eta}{2}\right)\] 
\end{lemma}
\begin{proof}
	We first show that for all $(x,y)\sim D$ such that $x \in \Xsafe_t$,
	$$\phi_t(x,y) - \phi_{t+1}(x,y) \geq  \lambda\mu_t(x,y) (y h_t(x) - 2\lambda) .$$
	We prove this statement for examples such that $0 \leq yG_{t}(x), yG_{t+1}(x) < s$, i.e., in the non-constant region of $M$, and observe that this suffices to prove the statement for all $(x,y)$ such that $x \in \Xsafe_t$. To see that this is true, note that within the constant region of $M$, $\phi_t(x,y) - \phi_{t+1}(x,y) = \lambda\mu_t(x,y)y h_t(x)$. For examples moved by $h_t$ from constant to non-constant regions of $M$, 
	\begin{align*}
	\phi_t(x,y) - \phi_{t+1}(x,y) 
	&= \int_{yG_t(x)}^{0}1 dz + \int_{0}^{yG_{t+1}(x)} e^{-z}dz \\
	&\geq \int_{0}^{\lambda yh_t(x)} e^{-z}dz \\
	&\geq \lambda M(0)(yh_t(x) - 2\lambda) & (\text{by assumption} ) \\
	&= \lambda \mu_t(x,y)(yh_t(x) - 2\lambda)
	\end{align*}	
	Similarly, for examples moving into the constant region from non-constant,
	\begin{align*}
	\phi_t(x,y) - \phi_{t+1}(x,y) 
	&= -\int_{yG_{t+1}(x)}^0 1 dz - \int_{0}^{yG_{t}(x)} e^{-z}dz \\
	& = \int_{yG_{t}(x)}^{0} e^{-z}dz + \int_0^{yG_{t+1}(x)} 1 dz  \\
	&\geq \int_{yG_{t}(x)}^{yG_{t+1}} e^{-z}dz  \\
	&\geq \lambda \mu_t(x,y)(yh_t(x) - 2\lambda) & (\text{by assumption/proved below} ) \\
	\end{align*}
	and so it only remains to prove the claim for examples such that $0 \leq yG_{t}(x), yG_{t+1}(x) < s$.
	By the definition of $\phi_t$, we have	
	\begin{align*}
	\phi_t(x,y) -& \phi_{t+1}(x,y) \\
	&= \int_{yG_t(x)}^{yG_{t+1}(x)} M(z) dz \\
	&= \int_{yG_t(x)}^{yG_{t+1}(x)} e^{-z} dz \\
	&= e^{-v}(yG_{t+1}(x) - yG_t(x)) &  (\text{for some } v\in [yG_t(x), yG_{t+1}(x)] \text{ by mean value theorem}) \\
	&                                                \geq  e^{-yG_{t+1}(x)}\lambda yh_t(x) \\
	& = e^{-yG_t(x)}e^{-\lambda yh_t(x)}\lambda y h_t(x) \\
	&\geq \mu_t(x,y)\lambda y h_t(x) - 2\mu_t(x,y) \lambda^2 & (xe^{-x} \geq x - 2x^2 \text{ for } x \in [-1,1]) \\
	&= \lambda \mu_t(x,y)(y h_t(x) - 2\lambda)
	\end{align*}
	and so the contribution to the potential drop from $(x,y) \in \Xsafe_t$ is as claimed. 
	
	We now consider the contribution to the potential drop from examples $(x,y)$ where $x \in \Xrisky_t$, by analyzing two complementary cases.
	\begin{enumerate}
		\item $\overconf^{\exor^{\mathrm{Mas}}(f, D_x, \eta(x))}(G_t, s, \delta_{\mathtt{err}}, \epsilon)$ returns false
		\item $\overconf^{\exor^{\mathrm{Mas}}(f, D_x, \eta(x))}(G_t, s, \delta_{\mathtt{err}}, \epsilon)$ returns true 
	\end{enumerate}

	In the first case, $h^{r}_t = 0$, and so $yG_{t+1}(x) = yG_{t}(x)$ holds for all these examples. Therefore the contribution to the potential drop is 
	\begin{align*}
	\E_{(x,y)\sim D} \big[\phi_t(x,y) - \phi_{t+1}(x,y) \big\vert x \in \Xrisky_t \big]
	& =  0 .
	\end{align*}

	In the second case, $h_t^c(x) = -\sgn(G_t(x))$, and so $yG_{t+1}(x) = yG_{t}(x) - \sgn(G_t(x))$ for these examples. $\overconf^{\exor^{\mathrm{Mas}}(f, D_x, \eta(x))}(G_t, s, \delta_{\mathtt{err}}, \epsilon)$ only returns true if it has estimated the error on examples such that $x \in \Xrisky_t$ exceeds $\eta + 3\epsilon/4$.  This routine estimates the error from a sample of size $8\log(2/\delta_{\mathtt{err}})/\epsilon^2$, and so it holds by the Chernoff-Hoeffding inequality that with all but probability $\delta_{\mathtt{err}}/2$, that
	$$\Pr_{(x,y)\sim D}\big[yG_t(x) \leq -s \big\vert x \in \Xrisky_t\big] \geq \eta + \epsilon/2 .$$ 
	This implies a contribution to the potential drop of
	\begin{align*}
	\E_{(x,y)\sim D} \left[\phi_t(x,y) - \phi_{t+1}(x,y) \big\vert x \in \Xrisky_t \right]
	& = \Pr_{(x,y)\sim D}[yG_t(x) \leq -s \big\vert x \in \Xrisky_t] \int_{yG_t(x)}^{yG_t(x)+ \lambda}1 dz   \\
	& \quad \quad + \Pr_{(x,y)\sim D}[yG_t(x) \geq s \big\vert x \in \Xrisky_t] \int_{yG_t(x)}^{yG_t(x)- \lambda}e^{-z} dz  \\
	& \geq (\eta + \epsilon/2)\lambda + (1 - \eta - \epsilon/2) \int_{yG_t(x)}^{yG_t(x)- \lambda}e^{-z} dz \\
	& \geq (\eta + \epsilon/2)\lambda + (1 - \eta - \epsilon/2)e^{-s}(1 - e^{-\lambda}) \tag{$yG_t(x) \leq s + \lambda$} \\
	& \geq (\eta + \epsilon/2)\lambda - (1 - \eta - \epsilon/2)e^{-s}(\lambda - \lambda^2)   \tag{$e^{-\lambda} \leq 1 - \lambda + \lambda^2$}\\
	& = (\eta + \epsilon/2)\lambda - (1 - \eta - \epsilon/2)(\frac{\eta + c}{1 - \eta})(\lambda - \lambda^2)  \tag{by definition of $s$}\\
	&\geq \frac{\epsilon\lambda}{2}(1 + \eta - \eta\lambda) - c\lambda(1 - \lambda) - \eta\lambda^2   \tag{from $\eta < \frac{\eta + c}{1 - \eta}$},      
	\end{align*}
	and so as long as $c \leq \epsilon/2 \leq \frac{\epsilon(1+ \eta - \eta\lambda)}{2(1 - \lambda)},$ we have 	
	$$\E_{(x,y)\sim D} [\phi_t(x,y) - \phi_{t+1}(x,y) \big\vert x \in \Xrisky_t ] \geq - \eta\lambda^2. $$
	Recall that we have assumed $\epsilon \geq \frac{8\eta\alpha}{1 - 2\alpha} = 2c$ and so the stated bound holds.
	
	We can now lower-bound the drop in the potential function. With probability $1 - \delta_{\mathtt{err}}$, $\estdens$ does not overestimate the error of the current hypothesis on points $x$ for which $x \in \Xrisky_t$ by more than $\epsilon/4$, and so we have
	
	\begin{align*}
	\Phi(t) - \Phi(t+1) 
	& = \E_{(x,y)\sim D} [\phi_t(x,y) - \phi_{t+1}(x,y)] \\
	& \geq \Pr_{(x,y)\sim D}[ x \in \Xsafe_t]\E_{(x,y)\sim D} [\lambda \mu_t(x,y)(y h_t(x) - 2\lambda) \big\vert x \in \Xsafe_t]  - \Pr_{(x,y)\sim D}[ x \in \Xrisky_t] \eta\lambda^2\\ 
	&= \Pr_{(x,y)\sim D}[ x \in \Xsafe_t]\E_{(x,y)\sim D} [\lambda \mu_t(x,y)(y h_t(x) - 2\lambda) \big\vert x \in \Xsafe_t] - \Pr_{(x,y)\sim D}[ x \in \Xrisky_t]\eta\lambda^2 \\
	& \quad \quad + \Pr_{(x,y)\sim D}[ x \in \Xrisky_t]\E_{(x,y) \sim D} [\lambda \mu_t(x,y)(y h_t(x) - 2\lambda) \big\vert x \in \Xrisky_t]  \quad (\mu_t(x,y) = 0 \text{ for } x \in \Xrisky_t)\\
	& \geq \E_{(x,y)\sim D} [\lambda \mu_t(x,y)(y h_t(x) - 2\lambda)] - \Pr_{(x,y)\sim D}[ x \in \Xrisky_t](\eta\lambda^2)\\
	& \geq  \E_{(x,y)\sim D} [\lambda \mu_t(x,y)(y h_t(x) - 2\lambda)] - \eta\lambda^2.
	\end{align*}	
	
	From our weak learner guarantee and Lemma~\ref{lem:wkl-rep}, we know that with all but probability $\delta_{\wkl}$, $h_t$ has advantage $\gamma/2$ against $D_{\mu_t}$. Therefore with all but probability $\delta_{\wkl} + \delta_{\mathtt{err}}$,	
	\[\Phi(t) - \Phi(t+1) \geq \frac{\lambda \gamma}{2} d(\mu_t) - 2\lambda^2 d(\mu_t) - \eta\lambda^2.\]

	Then taking $\delta_{\wkl} = \delta_{\mathtt{err}} = \frac{\delta\eta\gamma^2}{1536}$ and $\lambda = \gamma/8$, we have
	
	\[\Phi(t) - \Phi(t+1) \geq \frac{\gamma^2}{32}\left(d(\mu_t)  - \frac{\eta}{2}\right)\] 	
	with all but probability $\delta\eta\gamma^2/768$.
\end{proof}

Now we use our guaranteed drop in potential to show bounds on termination, as well as the density of the measure $\mu_t$ at the end of the final round $t$.
\begin{lemma}[Termination]\label{lem:termination}
	Let $\wkl$ be an $(\alpha, \gamma)$-weak learner requiring a sample of size $m_{\wkl}$ and let $\delta_{\wkl} = \delta\eta\gamma^2/1536$.
	Let $\lambda = \gamma/8$ and $\kappa \geq \eta$. Then with all but probability $\delta/3$, $\boost^{\wkl}$ terminates within $T \leq 128/(\eta\gamma^2)$ rounds, and conditioned on termination, $d(\mu_T) \leq \kappa + \epsilon/2$ with all but probability $\delta/8$.  
\end{lemma}

\begin{proof}
$\boost$ terminates once $\estdens$ estimates $\dest(\mu_t) \leq \kappa$. Given that $\estdens$ draws a sample of size $2\log(1/\delta_{\mathtt{dens}})/\beta^2$ for $\beta = \min\{\epsilon/2, \eta/4\}$, the Chernoff-Hoeffding inequality bounds  the probability that $\estdens$ overestimates $d(\mu_t)$ by more than $\beta$ by $\delta_{\mathtt{dens}}$. Therefore the probability that $\boost$ fails to terminate at the end of any round for which $d(\mu_t) \leq \kappa - \beta$ is no more than $\delta_{\mathtt{dens}}$. We condition on this failure not occuring for the rest of the proof.
    
From Lemma~\ref{lem:potential}, we have that with probability at least $1 - \delta_{\wkl} - \delta_{\mathtt{err}},$

\[\Phi(t) - \Phi(t+1) \geq \frac{\gamma^2}{32}\left(d(\mu_t) - \frac{\eta}{2}\right).\]
 We have taken $\kappa \geq \eta$, and $\beta \leq \eta/4$, so except with probability $\delta_{\wkl} + \delta_{\mathtt{err}}$, the potential drops by at least $\frac{\gamma^2}{32}(\kappa - \beta - \frac{\eta}{2}) > \frac{\eta\gamma^2}{128}$ in each round. The potential function begins at
\[\Phi_0 = \E_{(x,y)\sim D}\int_{0}^{\infty}M(z)dz = 1\]
and has minimum value 0, so taking $T = \frac{128}{\eta\gamma^2}$, it must be the case that $d(\mu) \leq \kappa - \beta$ by round $T$ with probability at least $1 - T(\delta_{\wkl} + \delta_{\mathtt{err}})$.  So with all but probability $128(\delta_{\wkl} + \delta_{\mathtt{err}})/(\eta\gamma^2)$, $d(\mu) \leq \kappa - \beta$ after $T$ rounds, and so $\boost$ must have terminated by then except with probability $\delta_{\mathtt{dens}}$. This gives a total failure probability of 
$$ \delta_{\mathtt{dens}} + \frac{128(\delta_{\wkl} + \delta_{\mathtt{err}})}{(\eta\gamma^2)} = \delta_{\mathtt{dens}} + \delta/6 \leq \delta/3 $$

It remains to bound the probability that $\boost$ terminates at round $t$ with $d(\mu_t) > \kappa + \epsilon/2$. Again arguing from the Chernoff-Hoeffding inequality and the sample size of $\estdens$, if $d(\mu_t) > \kappa + \epsilon/2$, $\boost$ terminates with probability no more than $\delta_{\mathtt{dens}}$. Union bounding over all rounds gives a failure probability $128\delta_{\mathtt{dens}}/(\eta\gamma^2) = \delta/8$.
\end{proof}

\subsection{Error Bounds}\label{sec:err-bounds}
In this subsection, we prove upper-bounds for the error of the final hypothesis $\Hyp = \sgn(G)$. We first prove an upper bound with respect to the distribution $D$ and then with respect to the target function $f$ on the marginal distribution $D_x$.
\begin{lemma}[Label error]\label{lem:label-err}
	When the algorithm terminates at round $t$, with all but probability $\delta/4$ over the randomness of $\boost$'s oracles and subroutines,
	
	\[\lerr(\Hyp) \leq \kappa + \epsilon  \]
\end{lemma}
\begin{proof}
	The proof proceeds by first bounding the error on safe examples by $\kappa + \epsilon/2$, and then arguing that our final hypothesis must either have low error on risky examples also, or the total probability mass assigned by $D$ to risky examples must be less than $\epsilon/2$.
	
	We begin by bounding the error on examples $(x,y) \in \Xsafe_t$. For all $(x,y) \in \Xsafe_t$, $\Hyp(x) \neq y$ if and only if $yG_t(x) \leq 0$, and therefore $\mu_t(x,y) = 1$. For all other examples, the measure $\mu_t(x,y) \geq 0$. From Lemma~\ref{lem:termination}, we have that with all but probability $\delta/8$, $d(\mu_t) \leq \kappa + \epsilon/2$ upon termination. Conditioning on this event and considering the minimum contribution to the density by all examples misclassified by $\Hyp$, we have
	
	\begin{align*}
	\kappa + \epsilon/2
	&\geq \E_{(x,y) \sim D} \mu(x,y) \\
	&= \sum_{\substack{(x,y) : \\ \Hyp(x) = y}}D(x,y)\mu(x,y) + \sum_{\substack{(x,y) :\\ \Hyp(x) \neq y}}D(x,y)\mu(x,y) \\
	&\geq \sum_{\substack{(x,y) : \\ \Hyp(x) \neq y}}D(x,y) \\
	&= \Pr_{(x,y) \sim D}[\Hyp(x) \neq y \big\vert x\in \Xsafe_t].
	\end{align*}
	
	Next, we bound the error of $\Hyp$ on examples $(x,y) \in \Xrisky_t$. By casework on the success of subroutine $\overconf$, we prove the following: 
	when the algorithm terminates at round $t$, with all but probability $\delta_{\mathtt{err}}$, at least one of the following holds. 
	\begin{enumerate}
		\item $\Pr_{(x,y) \sim D}[x\in \Xrisky_t] \leq \epsilon/2 $
		\item $\Pr_{(x,y)\sim D}[\Hyp(x)) \neq y \big\vert x\in \Xrisky_t] \leq \eta + \epsilon.$
	\end{enumerate}

	We first consider the case where $\overconf^{\exor^{\mathrm{Mas}}(f, D_x, \eta(x))}(G_t, s, \delta_{\mathtt{err}}, \epsilon)$ returns false in the last round of boosting.
	In this case, either the $\overconf$ routine estimated $\Pr_{(x,y) \sim D}[x \in \Xrisky_t] \leq \epsilon/4 $ or it estimated that $\Pr_{(x,y)\sim D}[\Hyp(x)) \neq y \big\vert x \in \Xrisky_t] \leq \eta + 3\epsilon/4$. The routine uses a sample of size $8\log(2/\delta_{\mathtt{err}})/\epsilon^2$ to estimate the probability that $x \in \Xrisky_t$, and so the probability of underestimating this quantity by more than $\epsilon/4$ is no more than $\delta_{\mathtt{err}}/2$, by the Chernoff-Hoeffding inequality. Similarly, the routine uses a sample of size $8\log(2/\delta_{\mathtt{err}})/\epsilon^2$ to estimate the error on examples such that $x\in \Xrisky_t$, and so underestimates this error by more than $\epsilon/4$ with probability no greater than $\delta_{\mathtt{err}}/2$. So if $\overconf$ returns false, at least one of the lemma's conditions hold with probability at least $1 - \delta_{\mathtt{err}}$.

If $\overconf$ returns true, then $|G_t(x)| = |G_{t-1}(x) + \lambda h^{s}_t(x)| - \lambda$. From Lemma~\ref{lem:invariant}, we know $|G_{t-1}(x)| < s + \lambda$ for all $x$, and $h^{s}_t(x) = 0$ for all $x$ such that $|G_{t-1}(x)| \geq s$. It follows that $x \in \Xsafe_t$ for all $x$, and so $\Pr_{(x,y) \sim D}[x\in \Xrisky_t] \leq \epsilon/2 $. 

Finally, we can bound the total error of $\Hyp$. The error bound for $x \in \Xsafe_t$ shows
\begin{align*}
\lerr(\Hyp) \leq \Pr_{(x,y)\sim D}[x \in \Xsafe_t](\kappa + \epsilon/2) + \Pr_{(x,y)\sim D}[x \in \Xrisky_t] \cdot \Pr_{(x,y) \sim D}[\Hyp(x) \neq y \big\vert x \in \Xrisky_t.]
\end{align*}
We have shown that with all but probability $\delta_{\mathtt{err}}$, either 
\[\Pr_{(x,y) \sim D}[x \in \Xrisky_t] \leq \epsilon/2 \] 
or 
\[\Pr_{(x,y)\sim D}[\Hyp(x) \neq y \big\vert x \in \Xrisky_t] \leq \eta + \epsilon.\]
In either case, since we took $\kappa \geq \eta$,
\[\lerr(\Hyp) \leq \kappa + \epsilon\]
with all but probability $\delta_{\mathtt{err}} + \delta/8 \leq \delta/4$.
\end{proof}

\begin{lemma}[Target function error]\label{lem:target-err}
	When the algorithm terminates, with all but probability 
	$\delta/4$,
	\[\ferr(\Hyp) \leq \frac{\kappa + \epsilon}{1 - \eta}\]
\end{lemma}
\begin{proof}
	Lemma~\ref{lem:label-err} shows that when the algorithm terminates, with all but probability 
	$\delta/4$,
	\[\lerr(\Hyp) \leq \kappa + \epsilon,\]
	so we consider the worst-case difference between misclassification error and target function error.

	\begin{align*}
	\kappa + \epsilon 
	&\geq \lerr(\Hyp) \\
	&= \Pr_{x \sim D_x}[\Hyp(x) \neq f(x)]\cdot \Pr_{(x,y) \sim D}[y = f(x) \big\vert \Hyp(x) \neq f(x) ] \\
	& \quad \quad \quad + \Pr_{x \sim D_x}[\Hyp(x) = f(x)]\cdot \Pr_{(x,y) \sim D} [y \neq f(x) \big\vert \Hyp(x) = f(x)] \\
	& \geq \Pr_{x \sim D_x}[\Hyp(x) \neq f(x)]\cdot \Pr_{(x,y) \sim D}[y = f(x) \big\vert \Hyp(x) \neq f(x) ] \\
	& \geq \Pr_{x \sim D_x}[\Hyp(x) \neq f(x)](1 - \eta) \\
	& = \ferr(\Hyp)(1 - \eta)
	\end{align*}
	and so $\ferr \leq \frac{\kappa + \epsilon}{1 - \eta}$
	with all but probability $\delta/4$.
\end{proof}

\subsection{Sample Complexity Analysis}\label{sec:samp-complexity}

In this subsection we give sample complexity bounds for the subroutines called by $\boost$, and the total sample complexity, for a single round of boosting. In all of the following lemmas, we assume that $\boost$ is being run with a $(\alpha, \gamma)$-Massart noise weak learner requiring a sample of size $m_{\wkl}$. 
As elsewhere, we use $\epsilon$ to denote the target error of the final hypothesis in excess of $\eta$, and use $\kappa$ to denote the density of $\mu$ below which $\boost$ terminates.
Let $\delta_{\mathtt{dens}}$ denote the probability that $\estdens$ fails to estimate the density of $\mu$ to within error $\beta = \min\{\epsilon/2, \eta/4\}$ and let $\delta_{\mathtt{err}}$ denote the probability that $\overconf$ fails to estimate the error of $G_t$ on examples $(x,y)$ such that $|G_t(x,y)| \geq s$.

\begin{lemma}[Sample complexity of $\samp$]\label{lem:samp-complexity}
	Let $\delta_{\wkl} = \delta\eta\gamma^2/1536$ and $\delta_{\mathtt{samp}} = \delta\eta\gamma^2/(1536\log(2/\delta_{\wkl}))$. With all but probability $\delta_{\mathtt{samp}}$, the $\samp$ routine draws no more than 
	$$m \in O\left( \frac{\log(1/\delta_{\samp})}{\kappa^2} + \frac{m_{\wkl}}{\kappa} \right)$$
	examples from $\exor^{\mathrm{Mas}}(f, D_x, \eta(x))$.
\end{lemma}
\begin{proof}
	Because $\boost$ terminates once the density of the measure $\mu$ is estimated to be less than $\kappa$, and, from the definition of $\beta$, we have that $$\log(1/\delta_{\mathtt{dens}})/(2\beta^2) \geq \max\{2\log(1/\delta_{\mathtt{dens}})/\epsilon^2, 8\log(1/\delta_{\mathtt{dens}})/\eta^2 \}$$ 
	many samples are used to estimate $d(\mu)$, it holds with all but probability $\delta_{\mathtt{dens}}$ that $d(\mu) \geq \kappa - \min\{\epsilon/2, \eta/4\}$.
	
	$\samp$ terminates once it has kept $m_{\wkl}$ examples, and so from Lemma~\ref{lem:rejsamp} we can conclude that 
	\begin{align*}
	m &= \frac{\log(1/\delta_{\samp})}{(\kappa - \min\{\epsilon/2, \eta/4\})^2} + \frac{2m_{\wkl}}{\kappa - \min\{\epsilon/2, \eta/4\}} \\
	& \in O\left(\frac{\log(1/\delta_{\samp})}{\kappa^2} + \frac{m_{\wkl}}{\kappa} \right)
	\end{align*}
	examples suffice except with probability $\delta_{\samp}$.
\end{proof}

\begin{lemma}[Sample complexity of testing weak hypotheses]\label{lem:samp-wkl}
	Let $\delta_{\wkl} = \delta\eta\gamma^2/1536$. With all but probability $3\delta_{\wkl},$ at most 
	$$ m \in O\left( \frac{\log(1/(\delta\eta\gamma))}{\kappa ^2} + \frac{m_{\wkl}\log(1/(\delta\eta\gamma))}{\kappa } + \frac{\log(1/(\delta\eta\gamma))}{\gamma^2\kappa}\right)$$
	examples from $\exor^{\mathrm{Mas}}(f, D_x, \eta(x))$ are drawn to identify a good enough weak hypothesis. 
\end{lemma}
\begin{proof}
	We have just shown in Lemma~\ref{lem:samp-complexity} that, with all but probability $\delta_{\mathtt{samp}}$, 
	$$m \in O\left( \frac{\log(1/\delta_{\samp})}{\kappa^2} + \frac{2m_{\wkl}}{\kappa}\right)$$ examples are required to draw a sample for $\wkl$. Recall from Definition~\ref{def:massart-weak-learner} that we assume $\wkl$ has failure probability 1/3, and from Lemma~\ref{lem:wkl-rep}, that we invoke $\wkl$ on $2\log(2/\delta_{\wkl})$ different samples to ensure we have at least one hypothesis with advantage $\gamma$, except with probability $\delta_{\wkl}/2$. To estimate which hypothesis is best, we draw $2\log(2/\delta_{\wkl})/\gamma^2$ examples from $D_{\mu}$, against which we test each hypothesis.	
	To draw these additional $2\log(2/\delta_{\wkl})/\gamma^2$ examples from $D_{\mu}$, with all but probability $\delta_{\mathtt{samp}}$, we make at most 
	$$m \in O\left(\frac{\log(1/\delta_{\mathtt{samp}})}{\kappa^2}  + \frac{\log(1/\delta_{\wkl})}{\kappa\gamma^2}\right)$$
	calls to $\exor^{\mathrm{Mas}}(f, D_x, \eta(x))$.
	
	We took  $\delta_{\wkl} = \delta\eta\gamma^2/1536$ and $\delta_{\mathtt{samp}} = \delta\eta\gamma^2/(1536\log(2/\delta_{\wkl}))$, so to repeatedly run the weak learner and identify a good enough hypothesis, we require 
	\begin{align*}
	m 
	&\in O\left(  \frac{\log(1/\delta_{\wkl})\log(1/\delta_{\samp})}{\kappa ^2} + \frac{m_{\wkl}\log(1/\delta_{\wkl})}{\kappa } + \frac{\log(1/\delta_{\wkl})}{\kappa\gamma^2}\right) \\
	&\in O\left( \frac{\log(1/(\delta\eta\gamma))}{\kappa ^2} + \frac{m_{\wkl}\log(1/(\delta\eta\gamma))}{\kappa } + \frac{\log(1/(\delta\eta\gamma))}{\gamma^2\kappa}\right)
	\end{align*}
	examples, except with probability 
	\begin{align*}
	\delta_{\mathtt{\samp}} + 2\log(2/\delta_{\wkl})\delta_{\mathtt{samp}} \leq 3\log(2/\delta_{\wkl})\delta_{\mathtt{samp}} = \frac{\delta\eta\gamma^2}{512} = 3\delta_{\wkl}.
	\end{align*} 
\end{proof}

\begin{lemma}[Sample complexity of $\overconf$]\label{lem:overconf-complexity}
	With all but probability $\delta\eta\gamma^2/768$,
	 $\overconf$ draws no more than 
	$$m \in O\left( \frac{\log(1/\delta\eta\gamma)}{\epsilon^3}\right)$$
	examples from $\exor^{\mathrm{Mas}}(f, D_x, \eta(x))$.
\end{lemma}
\begin{proof}
	$\overconf$ draws samples to estimate two population statistics: the probability that $x \in \Xrisky_t$ and, if that estimate exceeds $\epsilon/4$, the error of $G_t$ on examples such that $x \in \Xrisky_t$. 
	
	To estimate $\Pr_{(x,y) \sim D}[x \in \Xrisky_t]$ within error $\epsilon/8$ with all but probability $\delta_{\mathtt{err}}/2$, it draws $32\log(2/\delta_{\mathtt{err}})/\epsilon^2$ examples. 
	Then to estimate $\E_{(x,y)\sim D}[|y-\sgn(G_t(x))| \big\vert x \in \Xrisky_t]$ to within error $\epsilon/4$ with failure probability $\delta_{\mathtt{err}}/2$, it uses a sample of size $8\log(2/\delta_{\mathtt{err}})/\epsilon^2$, but requires that all these examples satisfy $x \in \Xrisky_t$. As we know $\Pr_{(x,y) \sim D}[x \in \Xrisky_t] \geq \epsilon/8$ with all but probability $\delta_{\mathtt{err}}/2$, another use of the Chernoff-Hoeffding inequality allows us to upper-bound by $2\delta_{\mathtt{err}}$ the probability that $\overconf$ draws more than 
	\begin{align*}
	m 
	& = \frac{64\log(1/\delta_{\mathtt{err}})}{\epsilon^2} + \frac{128\log(2/\delta_{\mathtt{err}})}{\epsilon^3}  \\
	& \in O\left(\frac{\log(1/\delta_{\mathtt{err}})}{\epsilon^3}\right) \\
	& \in O\left( \frac{\log(1/\delta\eta\gamma)}{\epsilon^3} \right)
	\end{align*}
	examples to estimate the error.
	
	Therefore with all but probability $2\delta_{\mathtt{err}} = \delta\eta\gamma^2/768$, $\overconf$ terminates having drawn no more than 
 	$$m \in O\left(\frac{\log(1/(\delta\eta\gamma)}{\epsilon^3}\right)$$
	examples. 
\end{proof}

\begin{lemma}[Sample complexity of one round]\label{lem:round-samp-complexity}
	With all but probability $5\delta\eta\gamma^2/1536$, one round of boosting with $\wkl$ draws no more than 
	$$m \in O\left(\frac{ \log(1/(\delta\eta\gamma)}{\min\{\epsilon, \eta\}^2} + \frac{\log(1/(\delta\eta\gamma)}{\epsilon^3} +  \frac{\log(1/(\delta\eta\gamma))}{\kappa ^2} + \frac{m_{\wkl}\log(1/(\delta\eta\gamma))}{\kappa } + \frac{\log(1/(\delta\eta\gamma))}{\gamma^2\kappa} \right)$$ 
	examples from $\exor^{\mathrm{Mas}}(f, D_x, \eta(x))$.
\end{lemma}
\begin{proof}
	In a single round of boosting, at most one call is made to $\estdens$ and $\esterr$ routines, and one weak hypothesis is chosen; no calls to the example oracle $\exor^{\mathrm{Mas}}(f, D_x, \eta(x))$ are otherwise made. The $\estdens$ procedure draws exactly $$\frac{\log(1/\delta_{\mathtt{dens}})}{2\min\{\epsilon/2, \eta/4\}^2} \in O\left(\frac{ \log(1/(\delta\eta\gamma)}{\min\{\epsilon, \eta\}^2}\right)$$
	examples. Lemma~\ref{lem:overconf-complexity} shows that, with all but probability $\delta\eta\gamma^2/768$, the $\overconf$ routine draws no more than 
	$$m  \in O\left(\frac{\log(1/(\delta\eta\gamma)}{\epsilon^3}\right)$$
	examples. Lemma~\ref{lem:samp-wkl} shows that, with all but probability $\delta\eta\gamma^2/512,$ at most 
	$$O\left( \frac{\log(1/(\delta\eta\gamma))}{\kappa ^2} + \frac{m_{\wkl}\log(1/(\delta\eta\gamma))}{\kappa } + \frac{\log(1/(\delta\eta\gamma))}{\gamma^2\kappa}\right)$$
	 examples are drawn to choose a weak hypothesis. 
	So with all but probability $\delta\eta\gamma^2(\tfrac{1}{768} + \tfrac{1}{512}) = 5\delta\eta\gamma^2/1536$, a single round draws no more than 
	\begin{align*}
	m 
	&\in O\left(\frac{ \log(1/(\delta\eta\gamma)}{\min\{\epsilon, \eta\}^2} + \frac{\log(1/(\delta\eta\gamma)}{\epsilon^3} +  \frac{\log(1/(\delta\eta\gamma))}{\kappa ^2} + \frac{m_{\wkl}\log(1/(\delta\eta\gamma))}{\kappa } + \frac{\log(1/(\delta\eta\gamma))}{\gamma^2\kappa} \right) 
	\end{align*}
	examples.  
\end{proof}


\subsection{Boosting Theorem}\label{sec:boost-theorem}
We can now put together the Lemmas of Section~\ref{sec:convergence}, Section~\ref{sec:err-bounds}, and Section~\ref{sec:samp-complexity} to prove our main result.

\begin{restatable}[Boosting Theorem]{thm}{boosting}\label{thm:boosting}
	Let $\wkl$ be an $(\alpha, \gamma)$-weak learner requiring a sample of size $m_{\wkl}$.
	Then for any $\delta \in (0,1/2]$, any Massart distribution $D$ with noise rate $\eta < 1/2$, and any $\epsilon \geq \frac{8\eta\alpha}{1 - 2\alpha}$, taking $\lambda = \gamma/8$ and $\kappa = \eta$, $\boost^{\wkl}(\lambda, \kappa, \eta, \epsilon, \delta, \gamma, \alpha, m_{\wkl})$ will, with probability $1 - \delta$,
	\begin{itemize}
		\item run for $T \in O\left(1/(\eta\gamma^2)\right)$ rounds
		\item output a hypothesis $\Hyp$ such that $\lerr(\Hyp) \leq \eta + \epsilon$ and $\ferr(\Hyp) \leq \frac{\eta + \epsilon}{1 - \eta}$
		\item make no more than $$	m \in O\left(\frac{\log(1/(\delta\eta\gamma))}{\eta\gamma^2\epsilon^3} +  \frac{\log(1/(\delta\eta\gamma))}{\eta^3\gamma^2} + \frac{m_{\wkl}\log(1/(\delta\eta\gamma))}{\eta^2\gamma^2 } + \frac{\log(1/(\delta\eta\gamma))}{\eta^2\gamma^4} \right)$$	calls to $\exor^{\mathrm{Mas}}(f, D_x, \eta(x))$
		\item run in time 
		$$ O\left(\frac{\log(1/(\delta\eta\gamma))}{\eta^2\gamma^4\epsilon^3} +  \frac{\log(1/(\delta\eta\gamma))}{\eta^4\gamma^4} + \frac{m_{\wkl}\log(1/(\delta\eta\gamma))}{\eta^3\gamma^4 } + \frac{\log(1/(\delta\eta\gamma))}{\eta^3\gamma^6} \right),$$
		neglecting the runtime of the weak learner.
	\end{itemize}
\end{restatable} 
	
\begin{proof}
		Lemma~\ref{lem:termination} shows that $\boost$ terminates within $T \in O\left(1/(\eta\gamma^2)\right)$ rounds, except with probability $\delta/3$. From Lemmas~\ref{lem:label-err} and ~\ref{lem:target-err}, we have that with all but probability $\delta/4$, $\lerr(\Hyp) \leq \kappa + \epsilon$ and $\ferr(\Hyp) \leq \frac{\kappa + \epsilon}{1 - \eta}$, so taking $\kappa = \eta$ gives
		\[\lerr(\Hyp) \leq \eta + \epsilon \]
		and
		\[\ferr(\Hyp) \leq \frac{\kappa + \epsilon}{1 - \eta}\] 
		
		To bound sample complexity, we recall Lemma~\ref{lem:round-samp-complexity} tells us that with all but probability 	 $5\delta\eta\gamma^2/1536$, one round of boosting with $\wkl$ draws no more than 
		$$m \in O\left(\frac{ \log(1/(\delta\eta\gamma))}{\min\{\epsilon, \eta\}^2} + \frac{\log(1/(\delta\eta\gamma)}{\epsilon^3} +  \frac{\log(1/(\delta\eta\gamma))}{\kappa ^2} + \frac{m_{\wkl}\log(1/(\delta\eta\gamma))}{\kappa } + \frac{\log(1/(\delta\eta\gamma))}{\gamma^2\kappa} \right)$$ 
		examples.
		We have taken $\kappa = \eta$, so union bounding the error probabilities over all $T \leq 128/\eta\gamma^2$ rounds of boosting gives us a sample bound of 
		\begin{align*}
		m \in O\left(\frac{\log(1/(\delta\eta\gamma))}{\eta\gamma^2\epsilon^3} +  \frac{\log(1/(\delta\eta\gamma))}{\eta^3\gamma^2} + \frac{m_{\wkl}\log(1/(\delta\eta\gamma))}{\eta^2\gamma^2 } + \frac{\log(1/(\delta\eta\gamma))}{\eta^2\gamma^4} \right)
		\end{align*}
		exceeded with probability no more than 
		$$\frac{128}{\eta\gamma^2}\cdot\frac{5\delta\eta\gamma^2}{1536} = \frac{5\delta}{12}.$$

		To prove the bound on overall runtime, we observe that the runtime of a single round of $\boost$, neglecting calls to the weak learner, is linear in the runtime of subroutines $\overconf$ and $\estdens$, and quasilinear in the runtime of $\samp$ (from repetition of $\wkl$). The runtime of each of these subroutines is dominated by computing $G_t(x)$ for each example drawn from $\exor^{\mathrm{Mas}}(f, D_x, \eta(x))$, either to decide membership of $x$ in $\Xrisky_t$ or to compute $\mu_t(x,y)$. The cost of evaluating $G_t$ is linear in $t$, and so from our round and sample complexity bounds, we have the total runtime over all $T \in \left(1/\eta\gamma^2\right)$ rounds $\boost$ is 
		\begin{align*}
		&O\left(T\left(\frac{\log(1/(\delta\eta\gamma))}{\eta\gamma^2\epsilon^3} +  \frac{\log(1/(\delta\eta\gamma))}{\eta^3\gamma^2} + \frac{m_{\wkl}\log(1/(\delta\eta\gamma))}{\eta^2\gamma^2 } + \frac{\log(1/(\delta\eta\gamma))}{\eta^2\gamma^4} \right)\right) \\
		& \in O\left(\frac{\log(1/(\delta\eta\gamma))}{\eta^2\gamma^4\epsilon^3} +  \frac{\log(1/(\delta\eta\gamma))}{\eta^4\gamma^4} + \frac{m_{\wkl}\log(1/(\delta\eta\gamma))}{\eta^3\gamma^4 } + \frac{\log(1/(\delta\eta\gamma))}{\eta^3\gamma^6}\right)
		\end{align*}

		Finally, we observe that the total probability of failure to achieve all of the claimed bounds is no more than $\frac{\delta}{3} + \frac{\delta}{4} + \frac{5\delta}{12} = \delta$, completing the proof.
\end{proof}

\subsection{Final Hypothesis $\Hyp$}\label{ssec:hypothesis}

This subsection contains some explanation of the structure of the final hypothesis $\Hyp$ output by our algorithm. We show that these hypotheses can be both efficiently represented and evaluated.  

$\boost$ maintains a function $G : \data \rightarrow \{\pm1\}$, initialized to the zero function $G_0(x) = 0$.
When $\boost$ terminates at round $t$, it outputs the classifier $\sgn(G_t)$.
$G_t$ can be computed from the threshold parameter $s$ and a length-$t$ sequence of pairs $((h_1, b_1), \dots, (h_t, b_t)) \in (\mathcal{\Hyp} \times \{0,1\})^t$, where $h_i$ is simply the weak learner hypothesis from round $i$, and $b_i = 1$ if $\overconf$ returned true at round $i$.   
$G_t(x)$ can then be efficiently computed by Routine~\ref{alg:compg}.

\begin{algorithm}[H]
	\floatname{algorithm}{Routine}
	\caption{$\mathtt{ComputeG}(x, s, (h_1, b_1), \dots, (h_t, b_t))$}\label{alg:compg}
	\begin{algorithmic}
		\STATE $\sigma = 0$ \\
		\FOR{$i \in \{1, \dots, t\}$}
		\IF{$|\sigma| < s$}
		\STATE $\sigma \gets \sigma + \lambda h_i(x)$
		\ELSE
		\IF{$b_i = 1$}
		\STATE $\sigma \gets \sigma - \lambda$
		\ENDIF
		\ENDIF
		\ENDFOR
		\RETURN $\sigma$
	\end{algorithmic}
\end{algorithm} 

Lemma~\ref{lem:termination} says we may assume $T \in \poly(1/\eta, 1/\gamma)$, so long as the weak learner's hypotheses can be efficiently represented and evaluated, $G_T$ can be as well, and of course $\Hyp = \sgn(G_T)$.

\section{Improved Round Complexity Analysis}\label{sec:optimizations}

In this section we revisit the round complexity of $\boost$. We show that a more careful use of the lower-bound on progress against our potential function (Lemma~\ref{lem:potential}) proves convergence in $O\left(\frac{\log^2(1/\eta)}{\gamma^2}\right)$ rounds, saving nearly a factor $\eta^{-1}$ in both time and sample complexity. 

Recall that Lemma~\ref{lem:potential} shows that in each round of $\boost$ we have 
	\[\Phi(t) - \Phi(t+1) \geq \frac{\gamma^2}{8}\left(d(\mu_t) - \frac{\eta}{2}\right)\] 
	for potential function 
	\[\Phi(t) = \E_{(x,y) \sim D} [\phi_t(x,y)] = \E_{(x,y) \sim D} \int_{yG_t(x)}^{\infty} M(z) dz. \]
For simplicity, Lemma~\ref{lem:termination} uses the fact that the algorithm terminates once it estimates $d(\mu) \leq \kappa$ to approximately lower-bound $d(\mu_t)$ by $\kappa$. Since we take $\kappa \geq \eta$, this lower-bounds the potential drop in each round by $O(\eta\gamma^2)$. 
However, this lower bound is loose at the beginning of the algorithm, when $d(\mu_0) = 1$. 
To use this observation to obtain a tighter analysis, we first lower-bound the density of the measure $\mu_t$ by the potential function $\Phi(t)$.
\begin{lemma}\label{lem:density-v-potential}
	For every round $t$, with all but probability $\delta_{\mathtt{err}} = \delta\eta\gamma^2/1536$ , $\frac{\Phi_t}{s + \lambda + 1} - 2(\eta + \epsilon) \leq d(\mu_t)$.
\end{lemma}

\begin{proof}
	To show $\frac{\Phi_t}{s + \lambda + 1} - 2(\eta + \epsilon) \leq d(\mu_t)$, we independently consider the contribution to the density from examples $(x,y) \in \Xsafe_t$ and $(x,y) \in \Xrisky_t$ as follows,
	\begin{align*}
	d(\mu_t) 
	&= \E_{(x,y)\sim D}[\mu_t(x,y)] \\
	&= \E_{(x,y)\sim D}[\mu_t(x,y) \big\vert x \in \Xsafe_t]\cdot \Pr_{(x,y) \sim D}[x \in \Xsafe_t] + \E_{(x,y)\sim D}[\mu_t(x,y) \big\vert x \in \Xrisky_t]\cdot \Pr_{(x,y) \sim D}[x \in \Xrisky_t].
	\end{align*}
	If $(x,y) \in \Xsafe_t$, one of two cases holds:
	\begin{enumerate}
	\item $ -s < yG_t(x) \leq 0$, so $\mu_t(x,y) = 1$ and $\phi_t(x,y) = -yG_t(x) + 1 \leq s + 1$ 
	\item $ 0 < yG_t(x) < s$, so $\mu_t(x,y) = \exp(-yG_t(x))$ and $\phi_t(x,y) = \exp(-yG_t(x))$
	\end{enumerate} 
	both of which imply $\mu_t(x,y) \geq \phi_t(x,y)/(s+1)$, and so 
	$$\E_{(x,y)\sim D}[\mu_t(x,y) \big\vert x \in \Xsafe_t] \geq \E_{(x,y)\sim D}\left[\frac{\phi_t(x,y)}{s+1} \big\vert x \in \Xsafe_t\right] \geq  \E_{(x,y)\sim D}\left[\frac{\phi_t(x,y)}{s+\lambda + 1} \big\vert x \in \Xsafe_t\right] - 2(\eta + \epsilon).$$
	If $(x,y) \in \Xrisky_t$, then we again have two cases to consider:	
	\begin{enumerate}
		\item $yG_t(x) \leq -s$, so $\mu_t(x,y) = 0$ and $\phi_{t}(x,y) = -yG_t(x) + 1 \leq s + \lambda + 1$
		\item $ yG_t(x) \geq s$, so $\mu_t(x,y) = 0$ and $\phi_t(x,y) = \exp(-yG_t(x)) \leq (\eta + c)/(1 - \eta)$.
	\end{enumerate}
	We observe that examples $(x,y)$ falling into case 2 satisfy 
	$$\mu_t(x,y) \geq \phi_t(x,y) - (\eta +c )/(1 - \eta) \geq \phi_t(x,y) - \frac{\eta + c}{1 - \eta},$$ 
	and in case 1, $\mu_t(x,y) \geq \phi_t(x,y)/(s + \lambda + 1) - 1$,	so to prove our lower-bound on $d(\mu_t)$, we must upper-bound $\Pr_{(x,y) \sim D}[yG_t(x) \leq -s]$.
	By the definition of Algorithm~\ref{alg:boost}, with all but probability $\delta_{\mathtt{err}}$ over the coins of $\overconf$, at the end of each round $t$ either $\Pr_{x \sim D_x}[ x \in \Xrisky_t] \leq \epsilon/2$ or $\Pr_{(x,y)\sim D}[ yG_t(x) \leq -s \mid x \in \Xrisky_t] \leq \eta + \epsilon$. 
	
	If $\Pr_{x \sim D_x}[ x \in \Xrisky_t] \leq \epsilon/2$, then this gives us
	\begin{align*}
	d(\mu_t) 
	& \geq \left(1 - \frac{\epsilon}{2}\right)\E_{(x,y) \sim D} \left[\frac{ \phi_t(x,y)}{s+ \lambda + 1} \big\vert x \in \Xsafe\right ] + \frac{\epsilon}{2}\E_{(x,y) \sim D} \left[\frac{\phi_t(x,y)}{s + \lambda + 1} - 1 \big\vert yG_t(x) \leq -s \right] \\
	& \geq \E_{(x,y)\sim D}\left[ \frac{\phi_t(x,y)}{s + \lambda + 1}\right] - \frac{\epsilon}{2} \\
	& \geq \frac{\Phi(t)}{s + \lambda + 1} - 2(\eta + \epsilon),
	\end{align*}
	and so the stated bound holds.
	
	If $\Pr_{(x,y)\sim D}[ yG_t(x) \leq -s \mid x \in \Xrisky_t] \leq \eta + \epsilon$, we have
	\begin{align*}
	\E_{(x,y)\sim D}[\mu_t(x,y) \big\vert x \in \Xrisky_t] 
	& =  \E_{(x,y) \sim D} \left[\frac{\phi_t(x,y)}{s + \lambda + 1} - 1 \big\vert yG_t(x) \leq -s \right]\cdot \Pr_{(x,y) \sim D}[yG_t(x) \leq -s \big\vert x \in \Xrisky_t] \\
	& \quad \quad \quad + \E_{(x,y)\sim D} \left[\phi_t(x,y) - \frac{\eta+c}{1 - \eta} \big\vert yG_t(x) \geq s \right]\cdot \Pr_{(x,y) \sim D}[ yG_t(x) \geq s \big\vert x \in \Xrisky_t] \\ 
	& \geq (\eta + \epsilon)\E_{(x,y) \sim D} \left[\frac{\phi_t(x,y)}{s + \lambda + 1} - 1 \big\vert yG_t(x) \leq -s \right] \\
	& \quad \quad \quad + (1 - \eta - \epsilon)\E_{(x,y)\sim D} \left[\phi_t(x,y) - \frac{\eta+c}{1 - \eta} \big\vert yG_t(x) \geq s \right] \\
	& \geq \E_{(x,y) \sim D} \left[\frac{\phi_t(x,y)}{s + \lambda + 1}  \big\vert x \in \Xrisky_t \right] - \eta - \epsilon - (1 - \eta - \epsilon)(\frac{\eta + c}{1 - \eta}) \\
	& \geq \E_{(x,y) \sim D} \left[\frac{\phi_t(x,y)}{s + \lambda + 1}  \big\vert x \in \Xrisky_t \right] - 2(\eta + \epsilon), 
	\end{align*}
	in which case it again holds that 
	
	\begin{align*}
	d(\mu_t) 
	&= \E_{(x,y)\sim D}[\mu_t(x,y) \big\vert x \in \Xsafe_t]\cdot \Pr_{(x,y) \sim D}[x \in \Xsafe_t] + \E_{(x,y)\sim D}[\mu_t(x,y) \big\vert x \in \Xrisky_t]\cdot \Pr_{(x,y) \sim D}[x \in \Xrisky_t] \\
	& \geq \left(\E_{(x,y)\sim D}\left[\frac{\phi_t(x,y)}{s+\lambda + 1} \big\vert x \in \Xsafe_t\right] - 2(\eta + \epsilon)\right)\cdot \Pr_{(x,y) \sim D}[x \in \Xsafe_t] \\
	& \quad \quad \quad + \left(\E_{(x,y) \sim D} \left[\frac{\phi_t(x,y)}{s + \lambda + 1}  \big\vert x \in \Xrisky_t \right] - 2(\eta + \epsilon)\right)\cdot \Pr_{(x,y) \sim D}[x \in \Xrisky_t] \\
	& = \frac{\Phi(t)}{s+\lambda + 1} - 2(\eta + \epsilon).
	\end{align*}
\end{proof}

Now that we have a lower-bound on $d(\mu_t)$ in terms of $\Phi(t)$, we can show faster convergence and prove the following theorem.
\begin{restatable}[(Improved) Boosting Theorem]{thm}{opt-boosting}\label{thm:opt-boosting}
Let $\wkl$ be an $(\alpha, \gamma)$-weak learner requiring a sample of size $m_{\wkl}$.
Then for any $\delta \in (0,1/2]$, any Massart distribution $D$ with noise rate $\eta < 1/2$, and any $\epsilon \geq \frac{8\eta\alpha}{1 - 2\alpha}$, taking $\lambda = \gamma/8$ and $\kappa = \eta$, $\boost^{\wkl}(\lambda, \kappa, \eta, \epsilon, \delta, \gamma, \alpha, m_{\wkl})$ will, with probability $1 - \delta$,
\begin{itemize}
	\item run for $T \in O\left(\log^2(1/\eta)/\gamma^2\right)$ rounds
	\item output a hypothesis $H$ such that $\lerr(H) \leq \eta + \epsilon$ and $\ferr(H) \leq \frac{\eta + \epsilon}{1 - \eta}$
	\item make no more than 
	$$m \in O\left(\frac{\log^2(1/\eta)}{\gamma^2}\left(\frac{\log(1/(\delta\eta\gamma))}{\epsilon^3} +  \frac{\log(1/(\delta\eta\gamma))}{\eta^2} + \frac{m_{\wkl}\log(1/(\delta\eta\gamma))}{\eta } + \frac{\log(1/(\delta\eta\gamma))}{\eta\gamma^2} \right)\right)$$	
	calls to $\exor^{\mathrm{Mas}}(f, D_x, \eta(x))$,
	\item run in time 
	$$m \in O\left(\frac{\log^4(1/\eta)}{\gamma^4}\left(\frac{\log(1/(\delta\eta\gamma))}{\epsilon^3} +  \frac{\log(1/(\delta\eta\gamma))}{\eta^2} + \frac{m_{\wkl}\log(1/(\delta\eta\gamma))}{\eta } + \frac{\log(1/(\delta\eta\gamma))}{\eta\gamma^2} \right)\right),$$	
	neglecting the runtime of the weak learner.
\end{itemize}
	
\end{restatable}

\begin{proof}
	It follows from Lemma~\ref{lem:dens-pot}, Lemma~\ref{lem:potential}, and Lemma~\ref{lem:density-v-potential} that 
	\begin{align*}
	\Phi(t+1) 
	&\leq \Phi(t) - \frac{\gamma^2}{32}\left(d(\mu_{t}) - \frac{\eta}{2}\right) \\
	&\leq \Phi(t) - \frac{\gamma^2}{32}\left(\frac{\Phi(t)}{s + \lambda + 1} - 2\eta - \frac{\eta}{2}\right)\\
	&\leq \Phi(t)\left(1 - \frac{\gamma^2}{64(s + 1)}\right) + \frac{\eta\gamma^2}{8} &  (\text{ from } s > \lambda ).
	\end{align*}
	Unrolling the recursion, we have that 
	\begin{align*}
	\Phi(t) &\leq \left(1 - \frac{\gamma^2}{64(s+1)}\right)^t + \frac{t\eta\gamma^2}{8} \\
	& \leq e^{-\gamma^2t/(64(s+1))} + \frac{t\eta\gamma^2}{8},
	\end{align*}
	and so taking $t = 64\log(1/\eta)(s+1)/\gamma^2$ and Lemma~\ref{lem:dens-pot} gives 
	\begin{align*}
	d(\mu_t) 
	&\leq \Phi(t) \\
	&\leq \eta + 8\eta\log(1/\eta)(s+1) \\
	&\leq \eta + 8\eta\log(1/\eta)(\log(1/\eta) + 1) \\
	&\leq \eta + 24\eta\log^2(1/\eta)
	\end{align*}
where the last inequality follows from $2\log(1/\eta) > 1$ for all $\eta < 1/2$. As we have already shown a potential drop of at least $\frac{\gamma^2\eta}{32}$ at each step for which $d(\mu) \geq \eta$, running for an additional $768\log^2(1/\eta)/\gamma^2$ rounds suffices to guarantee $d(\mu) \leq \eta = \kappa$. 
This gives a total round complexity of 
$$T \in O\left(\frac{\log^2(1/\eta)}{\gamma^2}\right).$$ 

The stated error bounds are the same as those proved in Theorem~\ref{thm:boosting}, and the tighter sample complexity and runtime follow immediately from the improved round complexity.  
\end{proof}

\section{Lower Bound on Error for Massart Boosting}
\label{sec:lower-bound}

In this section, we show that no ``black-box'' generic  boosting algorithm for Massart noise can have significantly better error than that of
our algorithm, $\eta + \Theta(\alpha \eta)$.
While the error term essentially matches the error lower bound of $\eta$ for RCN boosters from \cite{KalaiServedio:03}, it is unclear from their result whether generalizing to Massart noise should imply a lower bound of $\opt$ or a lower bound of $\eta$, since RCN is the special case of Massart noise where $\eta = \opt$.  
We show that the lower bound generalizes to the worst-case noise $\eta$, so long as $\opt$ is not negligible in the input size.
Therefore, no Massart-noise tolerant boosting algorithm can actually take advantage of a distribution with small expected noise to achieve accuracy better than its worst-case noise.

We consider the case where the target function $f \in \C$ is highly biased towards $-1$ labels (w.l.o.g.) and there is a small fraction of examples $(x,-1)$ where it cannot be distinguished whether $f(x) = 1$ and $\eta(x) = 0$, or $f(x) = -1$ and $\eta(x) > 0$.
As described in Section~\ref{ssec:techniques}, if the booster does not reweight the distributions on which it queries the weak learner to emphasize examples labeled $1$, an adversarial weak learner can return the constant function $-1$ and have high correlation. 
At the same time, if it does reweight its distribution to emphasize positively labeled examples, it risks violating the Massart condition by assigning to some $x \in \data$ a probability of appearing with its noisy label $y=-f(x)$ that is greater than $1/2 - \alpha$. 

\begin{restatable}{theorem}{lowerboundsimplethm}
	\label{thm:lower-bound-simple}
	If one-way functions exist, then no black-box Massart noise-tolerant boosting algorithm achieves label error $\eta + o(\alpha \eta)$, even when $\opt \ll \eta$. 
\end{restatable}

We formalize the notion of black-box boosting and review definitions in Section~\ref{ssec:lower-bound-prelims}. 
We describe the hard learning problem for the lower bound in Section~\ref{ssec:lower-bound-setting-parameters}.
We describe our adversarial weak learner and note its useful properties in Section~\ref{ssec:adversarial-weak-learner-and-example-generator}.
In Section~\ref{ssec:lower-bound-final-proof}, we state and prove our lower bound.

\subsection{Lower Bound Preliminaries}
\label{ssec:lower-bound-prelims}

First, we define black-box boosting. In particular, we formalize the notion of a sampling procedure $\B$, the subroutine a boosting algorithm uses to construct weak learner queries from labeled examples. 
Recall the definition of an efficient Massart noise weak learner from Section~\ref{sec:prelims}:
\weaklearnerdef*
\efficientweaklearnerdef*

We let $\wklcompiletime_\wkl$ denote the time $\wkl$ takes to output a hypothesis,
$\hypbitcomplexity$ denote the maximum bit complexity of a returned hypothesis $h$,
and $\hypevaltime$ denote the maximum time to evaluate a returned hypothesis $h$ on any $x \in \data$.
Recall that we define the runtime $R_\wkl$ of $\wkl$ as an upper bound on $\wklcompiletime_\wkl + \hypbitcomplexity + \hypevaltime$.

For a boosting algorithm to construct new distributions to query the weak learner,
the boosting algorithm must be able to convert examples from $D$ into examples from a new distribution. We refer to this part of the boosting algorithm as a \emph{sampling procedure} $\B$.  

\begin{definition}[Sampling Procedure]
	A sampling procedure $\B^{\mathcal{O}}$ is a probabilistic oracle algorithm that uses (potentially many) examples from $\mathcal{O}$ to return an example $(x,y) \in \data \times \pmone$.
\end{definition}

We prove a lower bound against the following formulation of a
black-box Massart boosting algorithm. 
In this setting, the boosting algorithm interacts with a example generator $\examplegenerator$, which generates examples for the weak learner. The boosting algorithm
provides $\examplegenerator$ an efficient sampling procedure $\B$, as well as oracle access to its example oracle $\exor$. The sampling procedure $\B$ will induced a new distribution over $\X \times \{\pm 1\}$. We denote by $D^{\B}$ the distribution induced by $\B$ when supplied with $\exor$ as its example oracle. The weak learner $\wkl$ uses $m_\wkl$ examples drawn i.i.d. by $\examplegenerator$ to compute a hypothesis $h$, returned to the boosting algorithm. Note that the weak learner is required to return a hypothesis with advantage $\gamma$ only if $D^\B$ is a Massart noise distribution with noise bound $1/2-\alpha$. 
For simplicity, we assume that the boosting algorithm and $\examplegenerator$ know the format of $h$ and $\B$, and that executing these subroutines can be done efficiently in their respective bit complexities.

\begin{definition}[Black-box Massart Boosting Algorithm]
	\label{def:black-box-boosting-algorithm}
Let $\C$ be a concept class over $\data$, and let $f \in \C$ be an unknown function.
Let $n$ denote the maximum bit complexity of an $x \in \data$. 
Let $D_x$ be a fixed but unknown distribution over $\data$.
Let $\exor = \exor(D_x, f, \eta(x))$ be a noisy example oracle for Massart noise distribution $D = \MD$.
Let $\examplegenerator$ be an example generator with query access to $\exor$. 
Let $\wkl$ be an efficient $(\alpha, \gamma)$-Massart noise weak learner with runtime $R_\wkl$, hypothesis bit complexity $\hypbitcomplexity$, and hypothesis evaluation time $\hypevaltime$. 
Let $m_\wkl$ denote the number of examples $\wkl$ requires. 
A black-box Massart boosting algorithm $\bbboost$,
with round bound $T$ and sample complexity $m$,
is a probabilistic polynomial-time algorithm with misclassification error $\eta^*$ if $\bbboost$
satisfies the following conditions:

\begin{enumerate}
\item Sample complexity $m$:
$\bbboost$ draws $m = \poly(n, 1/(1-2\eta), 1/\gamma)$ examples from sample oracle $\exor$. 
\item
Round bound $T$:
$\bbboost$ queries $\wkl$ at most $T = \poly(n, 1/(1-2\eta), 1/\gamma)$ times. 

\item
Weak Learner Queries:
$\bbboost$ queries $\wkl$ by providing input $\B$ to $\examplegenerator$,
where $\B^{\exor}$ is an efficient sampling procedure satisfying the following conditions:
\begin{itemize}
    \item $\B$ runs in time  $\poly(n, m_\wkl, 1/(1-2\eta), 1/\gamma, \hypevaltime)$. 
    \item $\B$ draws at most $\poly(n, 1/(1-2\eta), 1/\gamma)$ examples from $\exor$.
    \item $\B$ is represented with bit complexity $\poly(n, m_\wkl, 1/(1-2\eta), 1/\gamma, \hypevaltime)$. 
    \item $\B$ may use previous weak learner hypotheses as subroutines in $\B$.
\end{itemize}
$\examplegenerator(\B)$ runs $\B$ $m_\wkl$-many times to generate a sample $S$ containing $m_\wkl$ examples. $\examplegenerator$ gives $S$ to $\wkl$, which returns a hypothesis $h$ to $\bbboost$. 

\item 
Correctness:
If $\wkl$ returns a hypothesis with advantage $\gamma$ over $D^\B$ in each round that $D^\B$ is a Massart distribution, then
$\bbboost$ returns a classifier $H: \X \rightarrow \pmone$ with misclassification error $\lerr(H) < \eta^*$ with constant probability.

\item
Runtime:
$\bbboost$ runs in time $\poly(n, m_\wkl, 1/(1-2\eta), 1/\gamma, \hypevaltime)$. 
\end{enumerate}

\end{definition}

For clarity, the following pseudocode illustrates this black-box boosting framework.
\begin{figure}[H]
	\begin{algorithm}[H]
		\caption{Black-box Boosting Framework}
		\label{alg:bbboosting-framework}
		\begin{algorithmic}
			\STATE Black-box boosting algorithm $\bbboost$ draws $m$ examples from $\exor$.
			\FOR{$t = 1$ to $t = T$}
				\STATE $\bbboost$ constructs sample procedure $\B_t$, possibly using hypotheses $h_1, \dots, h_{t-1}$
				\STATE $\bbboost$ gives $\B_t$ to example generator $\examplegenerator$
				\STATE Weak learner $\wkl$ gives $m_\wkl$ to $\examplegenerator$
				\STATE $\examplegenerator$ uses $\B_t$ to draw $m_\wkl$ i.i.d. examples from $D^{\B_t}$. Let $S$ denote the set of these examples. 
				\STATE $\examplegenerator$ gives sample $S$ to $\wkl$
				\STATE $\wkl(S)$ returns hypothesis $h_t$ to $\bbboost$ 
			\ENDFOR
			\STATE $\bbboost$ outputs trained classifier $H$
		\end{algorithmic}
	\end{algorithm}
\end{figure}
The example generator $\examplegenerator$ is primarily used to correct the type mismatch between the boosting algorithm and weak learner. The boosting algorithm constructs distributions to query the weak learner, and the weak learner is defined to run on samples.

We will show that black-box Massart boosting algorithms cannot learn functions from pseudorandom function families with non-negligible probability.
The following definition appears in \cite{KalaiServedio:03}.
As noted in \cite{KalaiServedio:03}, if one-way functions exist,
then $p$-biased pseudorandom function families exist.

\begin{definition}[$p$-biased Pseudorandom Function Family]
	\label{def:pseudorandom-function-family]}
For $0 < p < 1$, a $p$-biased pseudorandom function family is a family of functions 
$\{f_s: \{0,1\}^{|s|} \rightarrow \pmone \}_{s \in \{0,1\}^*}$
which can be efficiently evaluated and satisfy the following
$p$-biased pseudorandomness property:
\begin{itemize}
    \item Efficient evaluation: 
    There is a deterministic algorithm which, given an $n$-bit seed $s$ and an $n$-bit input $x$, runs in time $\poly(n)$ and outputs $f_s(x)$.
    
    \item $p$-biased pseudorandomness: 
    Let $\F_{n,p}$ be the distribution over functions 
    from $\{0,1\}^n$ to $\pmone$
    such that function $F$ has weight $p^{|F^{-1}(1)|} (1-p)^{|F^{-1}(-1)|}$. 
    For all probabilistic polynomial time algorithms $\mathcal{A}$, the distinguishing advantage of $\mathcal{A}$ is a negligible function in $n$, 
    $$
    \left| 
    \Pr_{F \sim \F_{n,p}}[\mathcal{A}^F(1^n) \Rightarrow 1] 
    - \Pr_{s \sim \{0,1\}^n}[\mathcal{A}^{f_s}(1^n) \Rightarrow 1]
    \right|
    < \negl(n)
    $$
\end{itemize}
\end{definition}

\subsection{Adversarial Massart Distribution}
\label{ssec:lower-bound-setting-parameters}

Next, we describe the hard Massart noise learning problem used to prove our lower bound (Theorem~\ref{thm:lower-bound-detailed}). The following definitions apply to the remainder of Section~{\ref{sec:lower-bound}}.

Let $\eta \in [0, 1/2), \alpha \in (0, 1/2 - \eta), \gamma(\alpha) = \alpha/20$.
Define $\eta' = \eta (1 + \alpha/5)$.
Let $\{f_s: \{0,1\}^{|s|} \rightarrow \{\pm1\}\}_{s \in \{0,1\}^*}$
be a $\eta'$-biased pseudorandom random function family with minority label $1$. 

Let $n$ denote the security parameter, chosen to be at least a large polynomial in $1/(1-2\eta)$ and $1/\gamma$. 
Let $\data = \{0,1\}^n$, and let $D_x$ be the uniform distribution over $\data$.
For $s \in \{0,1\}^n$, let $\C_s$ be the concept class containing only the function 
$f_s : \{0,1\}^n \rightarrow \{\pm 1\}$. 

The noise function $\eta(x)$ is chosen as follows. 
On the minority elements $x \in f_s^{-1}(1)$, let $\eta(x) = 0$. 
On the majority elements $x \in f_s^{-1}(-1)$, let $\eta(x) = \eta$ for a random $\rho/(1-\eta')$-fraction of these $x$'s, where 
$1/\poly(n) < \rho < \alpha/1000$.
Later, we will refer to these elements as \emph{noisy}. Let $\Xnoise = \{x \in \X | \eta(x) > 0\}$ denote the set of noisy examples. 
For the remaining elements, let $\eta(x) = 0$. 
Finally, let Massart noise distribution $D = \mathtt{Mas}\{D_x, f_s, \eta(x)\}$ with example oracle $\exor = \exor(D_x, f_s, \eta(x))$. 
Note that the noise bound is $\eta$ and $\opt = \rho \eta$.

Throughout this section, we assume $n$ is a polynomial in $1/(1-2\eta)$ and $1/\gamma$, so that we can assume the probability of $\exor$ returning the same data point $x \in \X$ more than once during the $\poly(n, 1/(1 - 2\eta), \gamma)$ rounds of boosting is a negligible function in $n$.

\subsection{Adversarial Weak Learner and Example Generator}
\label{ssec:adversarial-weak-learner-and-example-generator}

In this section, we describe our adversarial weak learner $\wklr$, provide pseudocode, and prove that it has some nice properties. We also describe an example generator $\egr$ that does not directly call $\exor$.

\subsubsection{Adversarial Weak Learner}
\label{sssec:adversarial-wkl}
We now define our ``rude" weak learner $\wklr_{m, T}(S)$, which  attempts to be maximally unhelpful by returning hypotheses $h$ that rely entirely on majority vote labels. The weak learner $\wklr$ never provides the booster with any information about $f_s$ that the booster could not have computed itself, and therefore the pseudorandomness of $f_s$ will guarantee that the booster cannot boost $\wklr$ to obtain a hypothesis with error noticeably less than $\eta'$. The main technical challenge of proving our lower bound will come from showing that it is in fact possible for $\wklr$ to achieve noticeable advantage $\gamma$ against all Massart distributions supplied to it by the booster, without revealing any information about $f_s$ that cannot be efficiently simulated. 

Recall that boosting algorithm $\bbboost$ invokes the weak learner by constructing $\B$, an efficient sampling procedure, which induces a distribution $D^\B$. 
The weak learner $\wklr$ attempts to return a hypothesis $h: \data \rightarrow \pmone$ satisfying the following two conditions: 

\begin{itemize}
	\item For all $x \in \X$ that have large probability mass in $D^\B$ 
	($\approx \frac{\gamma}{10m}$ or larger), 
	$h(x)$ is the most likely label for $x$ under $D^{\B}$, i.e., $\sgn(\E_{(x^{*},y) \sim D^{\B}}[y \mid x = x^{*}])$. We will refer to such $x$'s as ``heavy-hitters". 
	
	\item For other $x$ with smaller probability mass in $D^{\B}$, $h(x)$ is the most likely label for all non-heavy-hitters under $D^{\B}$, i.e., $\sgn(\E_{(x,y) \sim D^{\B}}[y \mid x \not\in \Xset])$. 
	The weak learner $\wklr$ is given access to $m$, the number of examples drawn by the boosting algorithm, so that $\wklr$ may accurately predict which examples $x$ are heavy-hitters. 
\end{itemize}

\newcommand{\hhintervalwidth}{\ensuremath{[\frac{\gamma}{20m}, \frac{\gamma}{10m}]}}

The weak learner identifies heavy-hitters using a two-step process. 
First, $\wklr$ uses a subset of its sample $S$ to identify candidate heavy-hitters. It initially adds all $x$-values from this subset to the set of candidate heavy-hitters, $\Xset$.
Next, $\wklr$ checks each $x \in \Xset$ to see if it is indeed a heavy-hitter of $D^{\B}$. Fresh examples from its samples $S$ are used to empirically estimate this probability $\widehat{p}_x \eqdef \Pr_{(x', y') \sim S^\B}[x = x']$. 
The weak learner then randomly picks a value 
$v \in \hhintervalwidth$, 
and removes from $\Xset$ all $x$'s for which $\widehat{p}_x < v$. This step ensures that, with high probability, $\Xset$ contains \emph{exactly} $v$-heavy-hitters of
$D^\B$. 

The random choice of $v_h$ and $v_y$ will allow us to argue that, for fixed $v = (v_h, v_y)$, the hypothesis output by $\wklr$ is not too sensitive to the specific sample drawn by $\wklr$. That is, if $\wklr$ was repeatedly executed with the same choice of $v$, but different samples drawn from the same distribution, $\wklr$ would output the same hypothesis with high probability. This stability property is fully justified in Subsection~\ref{ssec:wklr-adv}, but, informally, it will allow us to argue that the booster could simulate the example oracle $\exor$ itself when generating samples for $\wklr$, without making additional queries to its example oracle, and that with high probability the hypotheses output by $\wklr$ would be the same in this case as those output when the sampling procedure queries $\exor$. Analyzing the behavior of the boosting algorithm when the sampling procedure does not draw examples from $\exor$ (and therefore the labels of examples do not depend on $f_s$) simplifies the argument that $\wklr$ can satisfy the definition of a Massart noise-tolerant weak learner without leaking information to the booster about $f_s$.

We now present pseudocode for our adversarial $(\alpha, \gamma)$-weak learner. 

\begin{figure}[H]
	\begin{algorithm}[H]
		\caption{$\wklr_{m, T}(S)$\\
		Precondition: $S$ contains $m_\wkl$ examples drawn i.i.d. from $D^\B$}
		\label{alg:wklr-pseudocode}
		\begin{algorithmic}
			\IF{$m < n$}
				\STATE $m = n$
			\ENDIF
			\STATE $\Xset \gets x$-values from $O(m^2/\gamma)$ examples from $S$
			\COMMENT{Step 1: Draw candidate heavy-hitters}
			\FORALL[Step 2: Remove non-$v$-heavy-hitters]{$x \in \Xset$}
			\STATE Estimate $\widehat{p}_x \eqdef \Pr[D^\B \text{ returns } x]$ using $O(m^{11}T^2/\gamma^4)$ fresh examples from $S$
			\ENDFOR
			\STATE $v_h \gets_r \hhintervalwidth$
			uniformly at random
			\STATE Remove from $\Xset$ all $x$ for which $\widehat{p}_x < v$.
			\STATE $v_y \gets_r [\tfrac{1}{2}, \tfrac{1}{2} +  \tfrac{\gamma}{10m}]$
			\FORALL[Step 3: Assign majority labels]{$x \in \Xset$}
			\STATE $S_x \gets m^2/(20\gamma^3)$ fresh examples from $S$
			\STATE $\widehat{p}_1 \gets $ fraction of $S_x$ with label $1$
			\IF{$\widehat{p}_1 \geq  v_y$}
			\STATE $y_x = 1$
			\ELSE
			\STATE $y_x = -1$
			\ENDIF 
			\ENDFOR           
			\STATE $h(x) = 
			\begin{cases}
			y_x & x \in \Xset \\
			-1     & \text{otherwise}
			\end{cases}
			$ 
			\COMMENT{Step 5: Output hypothesis $h$}
			\RETURN $h \eqdef \{\Xset, \{y_x\}\}$
		\end{algorithmic}
	\end{algorithm}
\end{figure}

This weak learner has polynomial sample complexity (Lemma~\ref{lem:wklr-sample-complexity}), runs in polynomial time (Lemma~\ref{lem:wklr-runtime}), and does not use any hardcoded information about $f_s$, so $\wkl$ is efficiently simulatable (Lemma~\ref{lem:wklr-egr-efficiently-simulatable}).

\subsubsection{Example Generation}
\label{sssec:adversarial-example-generator}

In this section, we define the two example generation procedures we will use in our lower bound argument: $\egh$ and $\egr$.  

Recall that an example generator $\examplegenerator$ is tasked with interfacing between the boosting algorithm, which creates reweighted distributions $D^\B$, and the weak learner, which runs on samples $S$ whose elements are drawn from $D^\B$. To accomplish this, the example generator needs information from the weak learner and the boosting algorithm. The weak learner tells the example generator $m_\wkl$, the sample size it needs, and the boosting algorithm provides oracle access to its example oracle $\exor$, as well as the sampling procedure $\B$. The example generator therefore invokes $\B$ $m_\wkl$-many times, returning sample $S$. 

\begin{figure}[H]
	\begin{algorithm}[H]
		\caption{$\egh_{m_\wkl}^{\exor}(\B)$ \\
		Precondition: $\B$ is a sampling procedure that returns an example $(x, y)$}
		\label{alg:honest-example-generator}
		\begin{algorithmic}
			\STATE $S = \emptyset$
			\FOR{$i = 1$ to $i = m_\wkl$}
			\STATE $(x,y) \gets \B^\exor$ 
			\STATE $S \gets S \| (x,y)$
			\ENDFOR
			\RETURN $S$
		\end{algorithmic}
	\end{algorithm}
\end{figure}

Our second example generation procedure $\egr$ behaves identically, except it never calls its oracle $\exor$. Rather, $\egr$ simulates calls to $\exor$ using $\exorsim$. The routine $\exorsim$ draws $x$ values from the same marginal distribution over $\X$ that $\exor$ does, $\mathcal{U}(\X)$. It then generates the label $y$ by taking $y=-1$ with probability $1-\eta'-\rho + \rho\eta$, and $y=-1$ otherwise, in effect sampling from the same marginal distribution over $\pm 1$ that $\exor$ does, but independent of the value $x$ it has already drawn, and therefore independent of $f_s$.

\begin{figure}[H]
	\begin{algorithm}[H]
		\floatname{algorithm}{Routine}
		\caption{$\exorsim$}
		\label{routine:exorsim}
		\label{alg:exorsim}
		\begin{algorithmic}
			\STATE $x\leftarrow_r U_n$
			\STATE $ y = 
			\begin{cases}
			-1 & \text{ w. p. } 1-\eta'-\rho + \rho\eta \\
			1 & \text{ o.w. }
			\end{cases}
			$				
			\COMMENT{i.e. $\Pr [y = 1] = \Pr_{(x,y) \sim D} [y = 1]$}
			\RETURN $(x,y)$
		\end{algorithmic}
		\end{algorithm}
	\begin{algorithm}[H]
		\caption{$\egr_{m_\wkl}^{\exor}(\B)$ \\
		Precondition: $\B$ is a sampling procedure that returns an example $(x, y)$}
		\label{alg:rude-example-generator}
		\begin{algorithmic}
			\STATE $S = \emptyset$
			\FOR{$i = 1$ to $i = m_\wkl$}
			\STATE $(x,y) \gets \B^\exorsim$ 
			\STATE $S \gets S \| (x,y)$
			\ENDFOR
			\RETURN $S$
		\end{algorithmic}
	\end{algorithm}
\end{figure}

By pseudorandomness, we will show that with high probability over $v = (v_h, v_y)$, and over choice of $S, S'$, where $S$ is generated by $\egh$ and $S'$ is generated by $\egr$, we have $\wklr(S; v) = \wklr(S'; v)$ (Section~\ref{sssec:wklr-lemmas-reproducibility1}). 

\subsubsection{Efficiency of $\wklr$ and $\egr$}
\label{sssec:wklr-lemmas}

In this section, we show that weak learner $\wklr$ and example generator $\egr$ are efficient and simulatable in polynomial time. 

\begin{enumerate}
	\item Efficiency of $\wklr$: polynomial sample complexity (Lemma~\ref{lem:wklr-sample-complexity}) and polynomial runtime (Lemma~\ref{lem:wklr-runtime}).
	
	\item Boosting with $\wklr$ and $\egr$ can be efficienctly simulated
	(Lemma~\ref{lem:wklr-egr-efficiently-simulatable}).
\end{enumerate}

Recall that $\B$ is a probabilistic algorithm that returns a labeled example. 
Let $m_\B$ denote the sample complexity of $\B$.
Let $R_\B$ denote the runtime of $\B$ (including the time to query its oracle). 

\begin{lemma}[Sample Complexity of $\wklr$]
		\label{lem:wklr-sample-complexity}
$m_\wklr = \poly(n, 1/(1-2\eta), 1/\gamma)$. 
\end{lemma}

\begin{proof}
By Definition~\ref{def:black-box-boosting-algorithm}, $m_\B = \poly(n, m, 1/(1-2\eta), 1/\gamma)$, $m = \poly(n, 1/(1-2\eta), 1/\gamma)$, and $T = \poly(n, 1/(1-2\eta), 1/\gamma)$. 
Step 1 requires $O(m^2T/\gamma)$ examples.
Step 2 requires $O(m^{13}T^2/\gamma^5)$ examples.
Step 3 requires $O(m^4/\gamma^4)$ examples. Therefore Step 2 dominates the sample complexity of the weak learner, and $m_\wklr = \poly(n, 1/(1-2\eta), 1/\gamma)$ as claimed. 
\end{proof}

\begin{lemma}[Runtime of $\wklr$]
	\label{lem:wklr-runtime}
	$\wklr$ runs in time  $R_\wklr = \poly(n, 1/(1-2\eta), 1/\gamma)$. 
	\begin{enumerate}
		\item $\wklr$ outputs hypothesis $h$ in time 
		$\wklcompiletime_\wklr = \poly(n, 1/(1-2\eta), 1/\gamma)$.
		\item The maximum bit complexity of $h$ is
		$\hypbitcomplexity = \poly(n, 1/(1-2\eta), 1/\gamma)$.
		\item Hypothesis $h$ can be evaluated in time
		$\hypevaltime = \poly(n, 1/(1-2\eta), 1/\gamma)$.
	\end{enumerate}
	
\end{lemma}

\begin{proof}
	By Definition~\ref{def:black-box-boosting-algorithm}, $m = \poly(n, 1/(1-2\eta), 1/\gamma)$, and $T = \poly(n, 1/(1-2\eta), 1/\gamma)$. Recall the runtime of a weak learner was defined as a bound on the sum of the three quantities listed in the lemma statement. 
	
	The hypotheses $h \eqdef \{\Xset, \{y_x\}\}$ output by $\wklr$ provides individual labels $y_x$ for a maximum of $O(m^2/\gamma)$ elements in $\Xset$. Thus, $\{\Xset, \{y_x\}\}$ has bit complexity at most
	$\poly(n, 1/(1-2\eta), 1/\gamma)$. The instructions for executing this hypothesis can also be written using $\poly(n, 1/(1-2\eta), 1/\gamma)$ bits. 
	For all $x \in \data$, $h(x)$ can be evaluated in time linear in the bit complexity of $h$. An algorithm can check if $x \in \Xset$ by scanning the representation of $h$ for $x$, outputting $y_x$ if found or $-1$ if not.
	Each step of $\wklr$ runs in time linear in the sample complexity of $\wklr$. 
	By Lemma~\ref{lem:wklr-sample-complexity}, $m_\wkl = \poly(n, 1/(1-2\eta, 1/\gamma)$. 
	Thus, $\wklr$ outputs $h$ in time $\poly(n, 1/(1-2\eta, 1/\gamma)$.
\end{proof}

Next, we argue that black-box boosting with $\wklr$ is efficiently simulatable. The following Lemma permits us to apply use the boosting algorithm in a distinguisher for pseudorandomness. 

\begin{lemma}[Boosting with $\wklr$ and $\egr$ can be efficiently simulated]
	\label{lem:wklr-egr-efficiently-simulatable}
	Given query access to a function oracle for $f_s$, a probabilistic algorithm $\mathcal{A}$ can simulate the interaction between $\bbboost$ and the weak learner $\wklr$, using $\egr$ to generate samples for $\wklr$, in time $\poly(n, 1/(1-2\eta), 1/\gamma)$.
\end{lemma}

\begin{proof}
	To simulate the initial $m$ examples drawn by the booster, $\mathcal{A}$ simulates $\exor$ as follows. It draws a data point $x\in X$ uniformly at random from $\X$, and queries its function oracle on this point. If the label returned by the function oracle is $1$, $\mathcal{A}$ returns $(x,1)$. If the label is a $-1$, it will return $(x, 1)$ with probability $\rho \eta$, and $(x,-1)$ otherwise. Because $\mathcal{A}$ only has negligible probability of drawing the same $x$-value twice, and because the noise function $\eta(x)$ is both random and non-zero only on a $\rho$-sized fraction of negatively-labeled examples, the $m$ examples drawn by this procedure are computationally indistinguishable from $m$ examples drawn from $\exor$, and so $\mathcal{A}$ successfully simulates the initial sample for $\bbboost$. 
	The algorithm $\mathcal{A}$ can then run the algorithm $\bbboost$, which by Definition~\ref{def:black-box-boosting-algorithm}, runs in time 
	$R_{b} = \poly(n, m_\wklr, 1/(1-2\eta), 1/\gamma, \hypevaltime)$. 
	
	To simulate samples generated by $\egr$, $\mathcal{A}$ can simply run Algorithm~\ref{alg:rude-example-generator}, using Routine~\ref{routine:exorsim} for the oracle to the sampling procedure $\B$. We have just shown in Lemma~\ref{lem:wklr-sample-complexity} and Lemma~\ref{lem:wklr-runtime} that $m_\wklr$ and $\hypevaltime$ are both $\poly(n, 1/(1-2\eta), 1/\gamma)$, and because the weak learner uses no special hard-coded information about $f_s$, it can also be efficiently simulated by $\mathcal{A}$. 
	These steps are repeated for $T = \poly(n, 1/(1-2\eta), 1/\gamma)$ rounds of boosting, each of which is efficiently simulatable in time $\poly(n, 1/(1-2\eta), 1/\gamma)$. Any additional post-processing must also be efficiently simulatable, since $\bbboost$ is assumed to run in time $\poly(n, m_{\wkl}, 1/(1 - 2\eta), 1/\gamma, R_h)$, and we have just shown that both $m_{\wkl}$ and $R_h$ are $\poly(n, 1/(1-2\eta), 1/\gamma)$. Therefore the entire interaction can be simulated by a probabilistic algorithm $\mathcal{A}$ with a function oracle for $f_s$, in time $\poly(n, 1/(1-2\eta), 1/\gamma)$.
	
\end{proof}

\subsubsection{$\wklr$ is a Massart Noise-Tolerant Weak Learner}\label{ssec:wklr-adv}

In this section, we analyze the advantage guarantee of $\wklr$. We begin by proving that the hypotheses output by $\wklr$ satisfy some notion of reproducibility. We then use this property, along with pseudorandomness of $f_s$, to argue that with high probability over choice of sample $S$ generated by $\egh$, $S'$ generated by $\egr$, and randomness $v = (v_h, v_y)$, that $\wklr(S; v) = \wklr(S', v)$. We then show that the hypothesis generated by $\wklr$, when run on a sample generated by $\egr$, will have good advantage against $D^{\B}$. Therefore the hypothesis generated by $\wklr$ during a real run of the boosting algorithm, where the sample is generated by $\egh$, must also have good advantage against $D^{\B}$.

\paragraph{Reproducibility of Weak Hypotheses.}
\label{sssec:wklr-lemmas-reproducibility1}

Recall that $\wklr$ and $\egr$ (or $\egh$) utilize randomness in two ways: i) to draw the input sample $S$, and ii) to pick thresholds $v_h$ and $v_y$. 
Let $h^{v}$ be the hypothesis that is most often returned when $\wklr$ is run with thresholds $v = (v_h, v_y)$.  
In this section, we show that weak learner $\wklr$ has the following property we call 
\emph{reproducibility}:
for a fixed $v$, with high probability over $S$, the hypothesis output by $\wklr$ is exactly $h^{v}$.
We will refer to this hypothesis as the 
\emph{canonical $v$-hypothesis} of $\wklr$. 

First, we will show that $\wklr$ run with $\egh$ is reproducible. 
In the next section, we apply pseudorandomness to show that with high probability, $\wklr$ run with $\egr$ outputs the same canonical $v$-hypothesis as it does when run with $\egh$.

Recall that $\wklr$ (Algorithm~\ref{alg:wklr-pseudocode}) is designed to return a hypothesis $h$ that assigns majority vote labels to $v_h$-heavy-hitters of $D^\B$, where $v_h$ is randomly chosen in the interval $\hhintervalwidth$.  
\begin{definition}[Heavy-Hitter]
	\label{def:heavy-hitter}
	Let $D$ be a distribution over $\data$. 
	We call $z \in \X$ a $v$-\emph{heavy-hitter} of $D$ if
	$ \Pr_{(x,y) \sim D} [x = z] 
	> v$.
\end{definition}

Recall that $\wklr$ returns a hypothesis $h \eqdef \{\Xset, \{y_x\}, b\}$. First, we show the consistency of $\Xset$. 

\begin{lemma}[Consistency of $\Xset$; $\Xset$ is the set of $v$-heavy-hitters]
	\label{lem:heavy-hitters-identified}
	Let $D$ be any distribution over $\data$, and let sample $S$ be a set of $m_\wklr$ examples drawn i.i.d. from $D$. Then with probability $1 - O(\tfrac{1}{mT})$ over the choice of $S$ and $v_h \in \hhintervalwidth$ (Step 2 of Algorithm~\ref{alg:wklr-pseudocode}), the set $\Xset$ computed by $\wklr(S)$ is exactly the set of $v_h$-heavy-hitters of $D$.
\end{lemma}

\begin{proof}
	Recall that $\wklr$ constructs a candidate list of heavy-hitters $\Xset$ in Step 1 of Algorithm~\ref{alg:wklr-pseudocode}, and prunes that list in Step 2. 
	
	In Step 1, $\wklr_{m,T}(S)$ uses $O(m^2/\gamma)$ examples to produce the initial set $\Xset$.
	Let $x$ be a $v$-heavy-hitter. 
	The probability that $x \not \in \Xset$ by the end of Step 1 is at most 
	$$
	(1-\gamma/(20m))^{m^2/\gamma}	< \exp(-m/20).
	$$
	Union bounding over the (at most) $20m/\gamma$ $v_h$-heavy hitters, the probability that the set $\Xset$ does not initially contain all $v$-heavy hitters is negligible in $m$.
	\newcommand{\hhapproxdistance}{\ensuremath{\frac{\gamma^3}{16000 m^4 T^2}}} 
	
	In Step 2, $\wklr$ estimates $\widehat{p}_x$ for each $x \in \Xset$ using $O(m^{11}T^2/\gamma^5)$ examples. The probability that a sample of this size contains fewer than $O(m^9T^2/\gamma^4)$ instances of $x$, given that $x$ is a heavy-hitter, is negligible in $m$, by a Chernoff-Hoeffding bound. Given this many instances of $x$, the probability that the estimate $\widehat{p}_x$ has error greater than $O(\gamma^2/(m^4T))$ is again a negligible function in $m$ by a Chernoff-Hoeffding bound. 
	Recall $\wklr$ chooses $v_h$ uniformly at random from the interval $\hhintervalwidth$.
	The probability that $v_h$ is chosen to be within distance
	$O(\gamma^2/(m^4T))$ 
	of the probability of a specific
	$\gamma/(20m)$-heavy-hitter of $D^\B$ is therefore no more than $O(\gamma/(m^3T))$.
	Union bounding over the at most $20m/\gamma$ heavy hitters, we have the following. Let $S_0, S_1$ be samples of size $m_{\wklr}$ drawn from $D^{\B}$. Denote by $\Xset_0(v_h)$ and $\Xset_1(v_h)$ the sets of $v_h$-heavy-hitters estimated by $\wklr$ provided $\Xset_0$ and $\Xset_1$ respectively. Then
	$$\pr_{\substack{S_0, S_1\\v_h\sim \hhintervalwidth}}[\Xset_0(v_h) \neq \Xset_1(v_h)] \in O(1/(m^2T)).$$

	It remains to show that, with high probability, all non-$v$-heavy-hitters are not included in $\Xset$ after Step 2. 
	There are at most $O(m^2/\gamma)$ candidate heavy hitters drawn in step 1. 
	With all but negligible probability in $m$, $\wklr$ estimates $\widehat{p}_x$ for all candidate heavy hitters to within error $O(\gamma^2/(m^4T))$. Then, as above, the probability that $v_h$ is chosen to be within distance
	$O(\gamma^2/(m^4T))$ 
	of the probability of a non-$v_h$-heavy-hitter of $D^\B$ is no more than $O(\gamma/(m^3T))$, and union bounding over the $O(m^2/\gamma)$ candidates gives probability $O(1/(mT))$. 
 	Therefore with probability $1 - O(\tfrac{1}{mT})$, at the end of Step 2, $\Xset$ contains exactly the $v_h$-heavy-hitters of $D^{\B}$. 
\end{proof}

Next, we show the consistency of $\{y_x\}$, the labels given by $\wklr(S)$ to $x \in \Xset$.

\begin{lemma}[Reproducibility of $h$ on heavy-hitters]
	\label{lem:heavy-hitters-correctly-labeled}
	Let $D$ be a distribution over $\data$, and let $S_0$ and $S_1$ be samples of $m_\wklr$ examples drawn i.i.d. from $D$. Denote by $h^v_0$ and $h^v_1$ the output of $\wklr(S_0; v_h = v)$ and $\wklr(S_1; v_h = v)$ respectively. Let $\Xset_0(v)$ and $\Xset_1(v)$ denote the respective sets of $v$-heavy-hitters computed by $\wklr(S_0; v_h = v)$ and $\wklr(S_1; v_h = v)$. Then we have
	$$\pr_{\substack{S_0, S_1\\v \sim \hhintervalwidth}}[\Xset_0(v) \neq \Xset_1(v) \text{ or } \exists x \in \Xset_0(v) \text{ s.t. } h_0(x) \neq h_1(x)] \in O(1/(mT)). $$ 
\end{lemma}

\begin{proof}
	
	By Lemma~\ref{lem:heavy-hitters-identified}, the probability that both  $\Xset_0(v)$ and $\Xset_1(v)$ are exactly the set of $v$-heavy-hitters of $D$ is at least $1 - O(1/mT)$, over the choice of $v, S_0,$ and $S_1$. 

	For each heavy-hitter $x$, $\wklr$ estimates the probability that $x$ has label $1$ in $D$ using $O(m^9T^2/\gamma^4)$ examples from $S$ (Step 3 of Algorithm~\ref{alg:wklr-pseudocode}). Given that $x\in \Xset$, the probability that this sample contains fewer than $O(m^7T^2/\gamma^4)$ instances of $x$ is negligible in $m$. By a Chernoff-Hoeffding bound, this estimate has error at most $O(\gamma^2/(m^3T))$ with all but negligible probability in $m$. 
	By an argument similar to the one of Lemma~\ref{lem:heavy-hitters-identified}, the probability that $v_y$ falls within $O(\gamma^2/(m^3T))$ of the true probability that $x$ is labeled $1$ in $D$ is $O(\gamma/(m^2T))$. Union bounding over the (at most) $20m/\gamma$ heavy-hitters proves the claim
	$$\pr_{\substack{S_0, S_1\\ v \sim \hhintervalwidth}}[\Xset_0(v) \neq \Xset_1(v) \text{ or } \exists x \in \Xset_0(v) \text{ s.t. } h_0(x) \neq h_1(x)] \in O(1/(mT)). $$ 
\end{proof}

Observing that $\wklr$ outputs the constant function $-1$ on all non-heavy-hitters, we have the following corollary. 
\begin{corollary}[Reproducibility of $h$]
	\label{cor:rep-of-h}
	Let $D$ be a distribution over $\data$, and let samples $S_0, S_1$ be two sets of $m_\wklr$ examples drawn i.i.d. from $D$. Let $h^v_0$ and $h^v_1$ denote the hypotheses output by $\wklr(S_0; v)$ and $\wklr(S_1; v)$ respectively. Then we have,  
	$$\pr_{S_0, S_1, v}[h^v_0 \neq h^v_1] \in O(1/(mT)).$$
\end{corollary}
We now use the reproducibility of $h$ and the pseudorandomness of $\{f_s\}$ to show that boosting $\wklr$ run with $\egr$ must also output the canonical $v$-hypothesis with high probability, unless $\pr_{x\sim D_x^{\B}}[x\not\in\Xset] < \gamma$.

\begin{lemma}[$\wklr$ does not distinguish between $\egh$ and $\egr$]
	\label{lem:reg-heg-equiv}
	Assume $\{f_s\}$ is a pseudorandom function family. 
	Let $\{\B_t\}_{t=1}^T$ be a sequence of sampling procedures constructed by the black-box boosting algorithm when boosting $\wklr$ for $T$ rounds. 
	Let $D^{\B_t}$ denote the distribution induced by $\B_t$ and the honest example generator $\egh$, and let $D^{\B_t}_{r}$ denote the distribution induced by $\B_t$ and the random example generator $\egr$. Let $S$ denote a sample of $m_{\wklr}$ examples drawn i.i.d. from $D^{\B_t}$, and let $S_r$ denote a sample of $m_{\wklr}$ examples drawn i.i.d. from  $ D_{r}^{\B_t}$. Let $h^v$ and $h^v_r$ denote the hypotheses output by $\wklr(S; v)$ and $\wklr(S_r; v)$ respectively. Then for all $t \in [T]$,
	$$\pr_{\substack{S, S_r\\ v }}[h^v_r \neq h^v ] \in O(1/(mT)).$$
\end{lemma}

\begin{proof}
	By Corollary~\ref{cor:rep-of-h}, we have that $\wklr(S)$ returns the canonical $v$-hypothesis $h^v$ for $D^{\B_t}$ with high probability over choice of $v$ and $S$. 
	Then if the claim does not hold, then it must be the case that there exists a round $t \in [T]$ such that, with probability $\omega(1/(mT))$ over choice of $v$, $S$, and $S_r$, we have ${\wklr(S; v) \neq \wklr(S_r; v)}$. Assuming this, we can construct the following distinguisher $\mathcal{A}$ against the pseudorandomness of $\{f_s\}$. 
	
	The distinguisher $\mathcal{A}$ executes the following procedure. It first chooses a round $t \in [T]$ uniformly at random, and simulates the interaction between the booster and $\wklr$ until round $t$. At round $t$, $\mathcal{A}$ draws a sample $S_0$ of $m_{\wklr}$ examples from $D^{\B_t}$ by simulating $\egh$. It then draws a sample $S_1$ of $m_{\wklr}$ examples by simulating $\egr$. It simulates $\wklr$ on both of these samples using the same choice of randomness $v$ for both simulations, and checks whether $\wklr(S_0; v) = \wklr(S_1; v)$. If not, it returns 1, and otherwise returns 0. 
	
	In the case that $\mathcal{A}$ is give oracle access to a random function $F$, both $S_0$ and $S_1$ are drawn from the same distribution, and so by Corollary~\ref{cor:rep-of-h}, $\wklr(S_0; v) = \wklr(S_1;v)$ with probability $1 - O(1/(mT))$ over the choice of $v$, and therefore $\mathcal{A}$ outputs 1 with probability $O(1/(mT))$.
	
	In the case that $\mathcal{A}$ is supplied a pseudorandom function $f_s$, by assumption there exists a round $t \in [T]$ at which ${\Pr_{\substack{S_0, S_1\\v}}[\wklr(S_0;v) \neq \wklr(S_1;v)] \in  \omega (1/(mT))}$. Therefore in this case, $\mathcal{A}$ outputs 1 with probability noticeably (in $n$) larger than in the random case, and so $\mathcal{A}$ is a distinguisher against the pseudorandomness of $f_s$. This is a contradiction, and therefore the claim holds. 	
	
\end{proof}

Informally, Lemma~\ref{lem:reg-heg-equiv} will allow us to construct distinguishing adversaries against the pseudorandomness of $f_s$ that make only $m$ queries of their function oracle. In the following lemmas, we will prove that $\wklr$ satisfies the definition of a Massart noise-tolerant weak learner when invoked on distributions constructed by the booster. That is, when $\wklr$ is given a sample from a Massart distribution generated by the boosting algorithm, it returns a weak hypothesis with advantage $\gamma$ with probability at least $2/3$. We will rely on appeals to the pseudorandomness of $f_s$ in these proofs, by showing that failure of $\wklr$ to return hypothesis with good advantage allows for the construction of distinguishers against the pseudorandomness of $f_s$. These distinguishers will simulate the boosting procedure, but it will be useful for our proofs to claim that the distinguishers can generate samples for $\wklr$ without making additional queries to their function oracles to generate labels for these samples. Lemma~\ref{lem:reg-heg-equiv} allows us to design distinguishers that use $\egr$ to generate samples for $\wklr$, rather than generating samples using $\egh$. Recall that $\egr$ makes no calls to the example oracle $\exor$, and simply generates labels randomly for examples drawn from the underlying marginal distribution $D_x$. Therefore we will assume that our distinguishers only query their function oracles for the purposes of simulating the first $m$ examples drawn by the booster. 

\paragraph{Advantage of $\wklr$.}

We will prove the following lemma by separately considering the advantage of weak hypotheses on heavy hitters of $D^{\B}$ and non-heavy hitters. 
\begin{restatable}[Advantage of $\wklr$]{lemma}{wklradvlem}
	\label{lem:wklr-adv}
	Let $D^{\B_t}$ denote the distribution induced by the sampling procedure $\B_t$ and $\egh$ at round $t \in [T]$ of boosting. Similarly, let $D_r^{\B_t}$ denote the distribution induced by $\B_t$ and $\egr$. Let $S_t$ denote a sample drawn i.i.d. from $D^{\B_t}_r$.
	Then for all $\poly(n, 1/(1 - 2\eta), 1/\gamma)$ rounds of boosting $\wklr$ with $\egr$, if $D^{\B}_t$ is Massart, then with probability $1 - O(1/(mT))$ over its internal randomness, $\wklr(S_t)$ outputs a hypothesis $h_t$ with advantage at least $\gamma$ against $D^{\B}_t$, except with negligible probability in $m$ over the choice of $\B_t$. 
\end{restatable}

Recall that $n$ is chosen to be a polynomial in $1/(1-2\eta)$ and $1/\gamma$, and $D_x$ is the uniform distribution over $\{0,1\}^n$. By birthday-paradox-style arguments, with all but negligible probability in $n$, no $x \in \data$ is output more than once by $\exor$ throughout boosting. Henceforth, we assume no $x \in \data$ is output more than once by $\exor$.

\begin{lemma}[$\wklr$ advantage against heavy-hitters of $D^{\B}$]
	\label{lem:wklr-adv-hh}
	Let $D^{\B_t}$ be the distribution induced by the sampling procedure $\B_t$ at round $t$. Similarly, let $D_r^{\B_t}$ denote the distribution induced by $\B_t$ and $\egr$. Let $S_t$ denote a sample drawn i.i.d. from $D^{\B_t}_r$, and let $h_t$ be the hypothesis output by $\wklr(S_t)$. Then for all $\poly(n, 1/(1 - 2\eta), 1/\gamma)$ rounds of boosting $\wklr$ with $\egr$, either
	\begin{enumerate}
		\item $\Pr \left[\tfrac{1}{2}\E_{(x,y) \sim D^{\B_t}}[h_t(x)y \mid x \in \Xset] \geq \alpha \right] \in 1 - O(1/(mT))$
		\item or $D^{\B_t}$ is not Massart.
	\end{enumerate}
\end{lemma}

\begin{proof}
	From reproducibility of $h_t$ (Lemma~\ref{lem:reg-heg-equiv}), we have that with probability $1 - O(1/(mT))$, $\wklr$ outputs the same hypothesis that it would have had it been given a sample from $D^{\B_t}$. For the remainder of the proof then, we will analyze the behavior of $\wklr$ given such a sample from $D^{\B_t}$, and show that it must have good advantage against the heavy-hitters of $D^{\B_t}$. 
	
	Suppose the second case does not hold, and therefore $D^{\B_t}$ is Massart. To compute $y_x$ for each $x\in \Xset$, $\wklr$ uses $m^2T/\gamma^3$ examples from $D_r^{\B_t}$. Because $x \in \Xset$, $D_x^{\B_t}(x) \geq \gamma/(40m)$ with high probability, and taking $\gamma =\alpha/20$, we have that at least $4m/\alpha^2$ instances of $x$ occur in $S_x$ (Step 3 of Algorithm~\ref{alg:wklr-pseudocode}) with all but negligible probability in $m$. The majority label of these $4m/\alpha^2$ examples is then taken to be the prediction of $h_t$ on $x$, which will agree with $f(x)$ with all but negligible probability in $m$, because we have assumed $D^{\B_t}$ is Massart, and so 
	$$\Pr_{(x,y)\sim D^{\B_t}}[y = f(x) \mid x \in \Xset] \geq 1/2 + \alpha.$$ 
	It then follows that 
	$$\Pr \left[\tfrac{1}{2}\E_{(x,y) \sim D^{\B_t}}[h_t(x)y \mid x \in \Xset] < \alpha \right] \leq \negl(m)$$ or $D^{\B_t}$ is not Massart, when $h_t$ is the hypothesis output by $\wklr$ given a sample $S$ from $D^{\B_t}$. Applying Lemma~\ref{lem:reg-heg-equiv} allows us to conclude that $$\Pr \left[\tfrac{1}{2}\E_{(x,y) \sim D^{\B_t}}[h_t(x)y \mid x \in \Xset] < \alpha \right] \leq O(1/(mT))$$ or $D^{\B_t}$ is not Massart, when $h_t \gets \wklr(S_t)$. 
\end{proof}

\begin{lemma}[$\wklr$ advantage against non-heavy hitters of $D^{\B}$]
	\label{lem:wklr-adv-nhh}
	Let $D^{\B_t}$ be the distribution induced by the sampling procedure $\B_t$ at round $t$. Similarly, let $D_r^{\B_t}$ denote the distribution induced by $\B_t$ and $\egr$. Let $S_t$ denote a sample drawn i.i.d. from $D^{\B_t}_r$, and let $h_t$ be the hypothesis output by $\wklr(S_t)$. Then for all $\poly(n, 1/(1 - 2\eta), 1/\gamma)$ rounds of boosting $\wklr$ with $\egr$, with all but negligible probability in $m$ over choice of $\B_t$, either
	\begin{enumerate}
		\item $\pr_{S_t, v} \left[\tfrac{1}{2}\E_{(x,y) \sim D^{\B_t}}[h_t(x)y \mid x \not\in \Xset] < \gamma \right] < 1/\poly(n)$
		\item $\pr_{(x,y)\sim D^{\B_t}} [x \not \in \Xset] < \gamma$ 
		\item or $D^{\B_t}$ is not Massart
	\end{enumerate}
\end{lemma}
\begin{proof}
	  
	Suppose that the first two conditions fail, implying that there exists a round $t$ of boosting for which the advantage of $h_t$ on the non-heavy hitters of $D^{\B_t}$ is less than $\gamma$, and that this will noticeably impact the overall advantage. Because $h_t$ takes a constant value $-1$ on all non-heavy hitters, it must then be the case that $$\Pr_{(x,y)\sim D^{\B_t}}[y = 1 \mid x \not\in \Xset ] > 1/2 - \gamma.$$ 
	Since we are considering the advantage only on examples such that $x\not\in\Xset$, then $D^{\B_t}(x,y) < \gamma/(10m)$ for all these examples. Furthermore, since we have assumed $\sum_{(x,y) : x \not \in \Xset} D^{\B_t}(x,y) \geq \gamma$, there must be at least $5m$ non-heavy-hitter examples $x$ such that $D^{\B_t}(x,1) > D^{\B_t}(x,-1)$ in order for $D^{\B_t}$ to satisfy $\Pr_{(x,y)\sim D^{\B_t}}[y = 1 \mid x \not\in \Xset ] > 1/2 - \gamma.$ 
	Then for $D^{\B_t}$ to be Massart, it must hold that $f(x) = 1$ for every example $x$ such that $D^{\B_t}(x,1) > D^{\B_t}(x,-1)$. However, if this is true with non-negligible probability in $n$, then we can construct the following distinguisher against $f_s$, which we denote by $\mathcal{A}$. 
	
	The distinguisher $\mathcal{A}$ simulates the boosting procedure run with $\wklr$ and $\egr$, as described in Lemma~\ref{lem:wklr-egr-efficiently-simulatable} up until round $t$, chosen uniformly at random from $[1,T]$. Once the boosting procedure reaches round $t$, $\mathcal{A}$ simulates the $t$th round of boosting and then queries its function oracle on all examples from the sample of the weak learner at that round that satisfy $D^{\B_t}(x,1) > D_x^{\B_t}(x)(1/2 - \alpha)$. If $f(x) = 1$ for all these examples, $\mathcal{A}$ outputs 1, and otherwise outputs 0. 
	
	To lower bound the advantage of our distinguisher, we will first show that there must be a significant number of examples drawn by $\wklr$ in round $t$ that satisfy $D^{\B_t}(x,1) > D_x^{\B_t}(x)(1/2 - \alpha)$. We begin by lower bounding the probability that this condition holds for a single non-heavy hitter example. 
	
	\begin{align*}
	\pr_{x \sim D_x^{\B_t}}[D^{\B_t}(x,1) > D_x^{\B_t}(x)(1/2 - \alpha)] 
	& = 1 - \pr_{x \sim D_x^{\B_t}}[D^{\B_t}(x,1) \leq D_x^{\B_t}(x)(1/2 - \alpha)] \\
	& = 1 - \pr_{x \sim D_x^{\B_t}}[D^{\B_t}(x,-1) \geq D_x^{\B_t}(x)(1/2 + \alpha)] \\
	& \geq 1 - \frac{1 + 2\gamma}{1 + 2\alpha} \\
	& = \frac{2(\alpha - \gamma)}{1 + 2\alpha} \\
	& \geq \alpha - \gamma \\
	& > \alpha/2,
	\end{align*}
	where the third line follows from the assumption that $\pr_{(x,y)\sim D^{\B_t}}[y=-1 \mid x \not\in\Xset] < 1/2 + \gamma$, and the last line follows from taking $\gamma = \alpha/20$. 
	Then because $\wklr$ has a sample of size $O(m^{13}T^2/\gamma^5)$, we have that the probability that fewer than $n$ of them satisfy $D^{\B_t}(x,1) > D_x^{\B_t}(x)(1/2 - \alpha)$ must be negligible in $m$ by the Chernoff-Hoeffding inequality.  
	
	We now proceed to bound the distinguishing advantage of $\mathcal{A}$, beginning with $\pr_{F \sim \mathcal{F}_{n,\eta'}}[\mathcal{A}^{F} \Rightarrow 1] $. In the case that $f$ is a random function, the boosting procedure can correctly identify an $x$ such that $f(x) = 1$ with probability no greater than $\frac{\eta'}{\eta' + (1 - \eta')\rho\eta}$. This follows immediately from taking the largest of the following conditional probabilities:
	\begin{align*}
	&\pr_{(x,y)\sim D^{\B_t}}[f(x) = 1 \mid y = 1] = \frac{\eta'}{\eta' + (1 - \eta')\rho\eta} \\
	&\pr_{(x,y)\sim D^{\B_t}}[f(x) = 1 \mid y = -1] = 0 \\
	&\pr_{(x,y)\sim D^{\B_t}}[f(x) = 1 \mid y \text{ unknown }] = \eta'.
	\end{align*}
	So if $f$ is a truly random function then the boosting procedure has probability no more than $(\frac{\eta'}{\eta' + (1-\eta')\rho\eta})^{n} $ of correctly identifying at least $n$ non-heavy hitter preimages of 1 under $f$. We have just shown that with all but negligible probability, $\wklr$ draws at least $n$ examples satisfying $D^{\B_t}(x,1) > D_x^{\B_t}(x)(1/2 - \alpha)$, and $\mathcal{A}$ returns 1 only if all of these examples are preimages of 1 under $f$. Therefore 
	$$\pr_{F \sim \mathcal{F}_{n,\eta'}}[\mathcal{A}^{F} \Rightarrow 1] \leq \left(\frac{\eta'}{\eta' + (1-\eta')\rho\eta}\right)^{n} + \negl(n), $$
	where the additive $\negl(n)$ term comes from the probability that fewer than $n$ qualifying examples were drawn by $\wklr$ in that round. 
	
	We now consider the case that $\mathcal{A}$ is provided $f_s$ as its oracle. Towards contradiction we have assumed that there exists some round $t$ at which, with probability $p$ that is non-negligible in $m$, the booster produces a Massart distribution $D^{\B_t}$ for which ${\Pr_{(x,y)\sim D^{\B_t}}[y = 1 \mid x \not\in \Xset] > 1/2 - \gamma  }$. Therefore with probability $1/T$ the distinguisher $\mathcal{A}$ will halt its simulation at this round, and so with probability $p/T$ will produce such a distribution. Then with all but negligible probability, it will draw $n$ examples such that $D^{\B_t}(x,1) \geq D_x^{\B_t}(x)(1/2 - \alpha)$. Since the distribution is Massart, all of these examples must satisfy $f(x) = 1$, and so we have
	$$\pr_{s \sim \{0,1\}^n}[\mathcal{A}^{f_s} \Rightarrow 1] = p/T - \negl(m), $$
	which is non-negligible in $m$, and therefore $n$. 
	Therefore $\mathcal{A}$ has distinguishing advantage
	\begin{align*}
	\pr_{s \sim \{0,1\}^n}[\mathcal{A}^{f_s} \Rightarrow 1] - \pr_{F \sim \mathcal{F}_{n,\eta'}}[\mathcal{A}^F \Rightarrow 1]  > \negl(n),
	\end{align*}
	which contradicts pseudorandomness of $\mathcal{F}_{n,\eta'}$. Therefore it must be the case that the boosting procedure only has negligible probability (in $m$) of generating a Massart distribution at any round that has at least $\gamma$ probability mass assigned to non-heavy hitters, and for which the constant function $-1$ does not have advantage at least $\gamma$ against non-heavy-hitters of $D^{\B_t}$. 
	
\end{proof}

We can now combine Lemma~\ref{lem:wklr-adv-hh} and Lemma~\ref{lem:wklr-adv-nhh} to show that $\wklr$, given a sample generated by $\egr$, will output a hypothesis with good advantage against $D^{\B_t}$.

\wklradvlem*

\begin{proof}
	
	The advantage of $h_t$ against $D^{\B_t}$ is $\tfrac{1}{2}\E_{(x,y) \sim D^{\B_t}}[yh(x)] $ where 
	
	\begin{align*}
	\E_{(x,y) \sim D^{\B_t}}[yh(x)] 
	&= \E_{(x,y) \sim D^{\B_t}}[yh(x) \mid x \in \Xset]\cdot \pr_{x \sim D_x^{\B_t}}[x \in \Xset] + \E_{(x,y) \sim D^{\B_t}}[yh(x) \mid x \not\in \Xset]\cdot \pr_{x \sim D_x^{\B_t}}[x \not\in\Xset]\\
	&\geq \alpha \cdot (1 - \pr_{x \sim D_x^{\B_t}}[x \not\in\Xset]) + \E_{(x,y) \sim D^{\B_t}}[yh(x) \mid x \not\in \Xset]\cdot \pr_{x \sim D_x^{\B_t}}[x \not\in\Xset] \\
	&= \alpha - \pr_{x \sim D_x^{\B_t}}[x \not\in\Xset]\cdot (\alpha - \E_{(x,y) \sim D^{\B_t}}[yh(x) \mid x \not\in \Xset] ),
	\end{align*}
	with all but probability $O(1/(mT))$, following from Lemma~\ref{lem:wklr-adv-hh}. From Lemma~\ref{lem:wklr-adv-nhh}, we have that if $D^{\B_t}$ is Massart, then with all but negligible probability, either $\E_{(x,y) \sim D^{\B_t}}[yh(x) \mid x \not\in \Xset] \geq \gamma$ or $\pr_{x \sim D_x^{\B_t}}[x \not\in\Xset] < \gamma$. Therefore $h$ has advantage at least $\gamma$ against $D^{\B_t}$ with probability at least $1 - O(1/(mT))$, and the claim holds. 
\end{proof}

\subsection{Lower Bound for Black-Box Massart Boosting~{\ref{thm:lower-bound-detailed}}}
\label{ssec:lower-bound-final-proof}

Finally, we prove that no black-box boosting algorithm can boost $\wklr$ to misclassification error better than $\eta(1+o(\alpha))$ with noticeable probability. At a high level, the proof idea is that any black-box booster interacting with $\wklr$ can be efficiently simulated, and so if a boosting algorithm was able to achieve misclassification error noticeably better than $\eta(1 + o(\alpha))$ for $\{f_s\}$, then there must be a distinguisher against the pseudorandomness of this function family, and so such error cannot be achievable via black-box boosting algorithms so long as pseudorandom functions exist.  

\begin{restatable}[Error Lower Bound Theorem]{theorem}{errorlowerboundthm}
	\label{thm:lower-bound-detailed}
	Let $\eta \in [0,1/2), \alpha \in (0,1/2 - \eta)$. Let $\{f_s\}$ be an $\eta'$-biased pseudorandom function family with security parameter $n$, where $\eta' = \eta(1+\alpha/5)$. Let $\eta$, $\alpha$ be at least inversely polynomially in $n$ bounded away from $1/2$. 
	Then, for random $s$, no efficient black-box boosting algorithm $\bbboost$ with example bound $m$
	running for
	$T$ 
	rounds, 
	given query access to $(\alpha, \gamma(\alpha) \eqdef \alpha/20)$-weak learner 
	$\wklr_{m, T}$ 
	and 
	$\poly(n, 1/(1-2\eta), 1/\gamma)$
	examples from example oracle $\exor(U_n, f_s, \eta(x))$, 
	can output a hypothesis with label error at most $\eta(1 + o(\alpha))$. 
	
	In particular,
	for all polynomials $q$, 
	for all polynomial time black-box Massart boosting algorithms $\bbboost$ with query access to $\wklr$
	and example oracle $\exor$,
	for $n$ sufficiently large,
	$$
	\Pr_{s \in U_n}
	\left[
	\mathrm{err}^{U_n, f_s}_{0\mhyphen1} (H)
	\le \eta'
	\right]
	< \frac{1}{q(n)}
	$$
	where $H$ is the trained classifier output by $\bbboost$.
\end{restatable}

\begin{proof}
	[Proof of Theorem \ref{thm:lower-bound-detailed}]
	\label{prf:lower-bound-detailed}

Let $\eta' = \eta(1+c\alpha)$. Suppose that $\bbboost$ achieves label error better than $\eta' - \epsilon$, for some noticeable $\epsilon$, and with noticeable probability $\delta$. Then we can construct a distinguisher $\mathcal{A}$ for $f_s$ as follows. 

The distinguisher $\mathcal{A}$ simulates the interaction between the booster and $\wklr$, where the samples for $\wklr$ are drawn by $\egr$ (as described in Lemma~\ref{lem:wklr-egr-efficiently-simulatable}). Once the booster outputs its final hypothesis $H$, $\mathcal{A}$ draws a set $S$ of $n/\epsilon^2$ elements from the uniform distribution over $\X$, restricted to examples on which it has not already queried its oracle. Because $\wklr$ is being run on samples drawn by $\egr$, $\mathcal{A}$ will only have simulated $\exor$, and therefore queried its oracle, for the $m$ examples used by the booster itself, and therefore $n/\epsilon^2$ elements can be drawn efficiently and the restricted distribution has only negligible statistical distance from $D_x$. The distinguisher $\mathcal{A}$ then queries both $H$ and its oracle on all elements of $S$, returning $1$ if its oracle and $H$ disagree on fewer than an $\eta' - \epsilon/2$ fraction of the elements, and $0$ otherwise. 

To show that $\mathcal{A}$ has non-negligible advantage distinguishing $f_s$ from a truly random function, we first consider the probability that $\mathcal{A}$ outputs $1$ when given oracle access to a truly random function, drawn from $\mathcal{F}_{n,\eta'}$. Because $\mathcal{A}$ is checking $H(x) \neq f(x)$ only on examples it has not previously queried, once $H$ is fixed, we have $\Pr_{x \sim \mathcal{U}(\mathcal{X})}[H(x) \neq f(x)  \mid x \text{ not previously queried} ] \geq \eta'$. Therefore

\begin{align*}
\pr_{F\sim \mathcal{F}_{n, \eta'}}[\mathcal{A}^F \Rightarrow 1 ] 
& = \pr_{\substack{F\sim \mathcal{F}_{n, \eta'} \\
		S \sim \mathcal{U}(\mathcal{X})}}[ \pr_{x \sim S}[H(x) \neq F(x)] \leq \eta' - \epsilon/2] \\
	& \leq \negl(n),
	\end{align*}
where the last line follows from a Chernoff-Hoeffding bound and the fact that $\mathcal{A}$ has drawn $n/\epsilon^2$ elements from $\X$ to check. 

We now consider the probability that $\mathcal{A}$ returns $1$ when given oracle access to pseudorandom $f_s$. We have assumed that our booster has noticeable probability $\delta$ of outputting a hypothesis $H$ with error less than $\eta' - \epsilon$, and from Lemma~\ref{lem:reg-heg-equiv}, we have that  
\begin{align*}
\pr_{s\sim \{0,1\}^n}[\mathcal{A}^{f_s} \Rightarrow 1 ] 
& = \pr_{\substack{s\sim \{0,1\}^n \\
		S \sim \mathcal{U}(\mathcal{X})}}[ \pr_{x \sim S}[H(x) \neq f_s(x)] \leq \eta' - \epsilon/2] \\
& \geq \delta(1 - \negl(n)).
\end{align*}

Since we have assumed $\delta$ is noticeable, and we have just shown that $\mathcal{A}$ has distinguishing advantage

$$ \pr_{s\sim \{0,1\}^n}[\mathcal{A}^{f_s} \Rightarrow 1 ]  - \pr_{F\sim \mathcal{F}_{n, \eta'}}[\mathcal{A}^F  \Rightarrow 1 ] > \delta/2, $$
the distinguisher $\mathcal{A}$ contradicts the pseudorandomness of $f_s$, and therefore $\wklr$ cannot be efficiently boosted to construct a hypothesis with error noticeably better than $\eta'$ with any noticeable probability.

\end{proof}

\section{Application: Massart Learning of Unions of High-Dimensional Rectangles}\label{sec:applications}


In this section, we exhibit a Massart weak learner for learning unions of rectangles. A direct application of Theorem~\ref{thm:boosting} yields an efficient Massart strong learner achieving misclassification error $\eta + \eps$. 
Recall that the Massart SQ lower bound of~\cite{CKMY20} applies to learning monotone conjunctions, ruling out efficient SQ algorithms with error $\opt+\eps$, even for a single rectangle. Furthermore, weak agnostic learning of a single rectangle is computationally hard in the agnostic model (see, e.g.,~\cite{FGRW09}).

\begin{definition}
A rectangle $B \in \mathbb{R}^d$ is an intersection of inequalities of the form
$x \cdot v < t$,   where $v \in \{\pm e_j : j \in [d]\}$ and $t \in \mathbb{R}$.
We may write a rectangle as a set $B$ of pairs $(v,t)$, that has size at most $2d$.
\end{definition}

We are interested in learning concepts $f \in C$ that are
indicator functions of unions of $k$ rectangles $B_1,\dots,B_k$. That is,
the class $\C$ consists of functions:
$$f(x) = \begin{cases}
+1 & \text{if } x \in \cup_{i \in [k]} \cap_{(v,t) \in B_i} [x \cdot v < t]\\
-1 & \text{otherwise }
\end{cases}$$
We refer to the negation of $\cup_{i \in [k]} B_i$ as the ``negative region".
Our weak learner aims to find if possible a rectangle entirely contained in the negative region 
to get some advantage over a random guess.
To this end, we establish a structural result which shows that unless an overwhelming part of the mass is positive, there always exists a rectangle with non-trivial mass that is contained in the negative region. Moreover this rectangle has a lot of structure as it consists of at most $k$ inequalities.

\begin{lemma}[Structural result]\label{lem:box_structural_result}
If the negative region has probability more than $\varepsilon$, 
there exists a rectangle contained in the negative region that has mass at least $\varepsilon / (2d)^k$. 
This rectangle can be written as an intersection of at most $k$ inequalities.
\end{lemma}

\begin{proof}
The negative region can be written as 
a union of $(2d)^k$ rectangles $B'$ with at most $k$ inequalities
$$\cup_{B' \in B_1 \times B_2 \times \cdots \times B_k} \cap_{(v,t) \in B'} [x.v \ge t]$$ 
by choosing which inequality is not satisfied in every rectangle.

Since the union of the rectangles covers is exactly the negative region and has mass at least $\varepsilon$, 
at least one rectangle $B'$ has probability more than $\varepsilon / (2d)^k$.
\end{proof}

\subsection{Weak Learner for Unions of Rectangles}\label{ssec:app-weak-learner}

Our weak learner exploits the structural result of Lemma~\ref{lem:box_structural_result} to obtain an advantage over a random guess. If the probability mass is overwhelmingly positive, then the hypothesis $h(x) = +1$ must correlate well with the observed labels. On the contrary, if there is sufficient negative mass, there must exist a rectangle where predicting $h(x) = -1$ correlates with the labels of the examples within that rectangle. This idea is presented in pseudo-code in $\wklbox$ and formalized in Lemma~\ref{lem:weak-learner} which gives the guarantees of our weak learner.

\begin{figure}[H]
\begin{algorithm}[H]
	\caption{$\wklbox^{\exor(f, D_x, \eta(x))}(d, k, \alpha)$}
	\label{alg:wkl-box}
	\begin{algorithmic}
		\STATE $S \gets \frac{k \, O(d)^{k}}{\alpha^2}$ examples from $\exor$
		\STATE $S^{-} \gets$ number of these examples labeled $-1$
		\IF{$\frac {|S^{-}|} {|S|} <  \frac \alpha 2$}
			\RETURN $h = +1$ \COMMENT{the constant $1$ hypothesis}
		\ELSE
			\FORALL {Rectangles $B$ = choice of $k$ examples and $k$ dimensions}
				\STATE $S_B \gets \{ (x,y) \in S | x\in B \}$
				\STATE $S^+_B \gets \{ (x,y) \in S | x\in B, y = +1\}$
				\STATE $B_{best} \gets $ $B$ that minimizes $|S^+_B| / |S_B|$ and has $|S_B| / |S| > \frac{\alpha} { 8 (2d)^k}$.
			\ENDFOR
			\STATE Let $z \in \{\pm1\}$ be the best most popular label in $S \setminus S_{B_{best}}$
			\STATE Hypothesis $h(x) = 
				\begin{cases} 
				-1 & x \in B_{best}\\
				z & \text{otherwise}
				\end{cases}$
			\RETURN $h$
		\ENDIF
	\end{algorithmic}
\end{algorithm}
\end{figure}

\begin{lemma}\label{lem:weak-learner}
The algorithm $\wklbox$ is a $(\alpha, \frac {\alpha^2} {O(d)^k})$-Weak Learner for unions of $k$ rectangles in $d$ dimensions.
It requires 
$k \frac {O(d)^{k}} {\alpha^2}$ 
samples and runs in time 
$\frac {k^k O(d)^{k^2+1}} {\alpha^{2 k}}$.
\end{lemma}

\begin{proof}
The algorithm starts by drawing drawing a set $S$ of $N = k \frac {O(d)^{k}} {\alpha^2}$ examples from $EX$.
Since the VC-dimension of rectangles defined by $k$ inequalities is $O(k)$ this guarantees that, with probability at least $2/3$, for any rectangle $B$, the empirical probabilities computed over the sample $S$ are close to actual ones:
\begin{itemize}
\item[1.] $|\Pr[ x \in B ] - \Pr_S[ x \in B ]| \le \alpha / O(d)^{k}$ 
\item[2.]  $|\Pr[ y = +1 \text{ and } x \in B' ] - \Pr_S[ y = +1 \text{ and } x \in B' ]| \le \alpha / O(d)^{k}$
\item[3.]  $|\Pr[ y = -1 ] - \Pr_S[ y = -1 ]| \le \alpha / O(d)^{k} \le \frac \alpha {4}$
\end{itemize}

Therefore, in the case that $|S^-| / |S| <  \frac \alpha {2}$, we have that $\Pr[ y = -1 ] <   \frac {3} {4} \alpha$.
Thus, the hypothesis $h = +1$ gets error at most $\frac {3} {4} \alpha + (\frac 1 2 - \alpha) \le \frac 1 2 -\frac \alpha {4}$.

Otherwise, there is at least $\frac \alpha {4}$ probability in the negative region. 
By Lemma \ref{lem:box_structural_result}, there is a rectangle $B^*$ defined by $k$ inequalities that is contained entirely in the negative region and has probability at least $\frac {\alpha} {4 (2d)^k}$. For this rectangle $B^*$ it holds that $\Pr[ x \in B^* ] \ge \frac {\alpha} {4 (2d)^k}$ and $\Pr[ y = +1 | x \in B ] \le \frac 1 2 - \alpha$. This means that within the sample $S$ it holds that
$\Pr_S[ x \in B^* ] > \frac {\alpha} {8 (2d)^k}$ and $\Pr_S[ y = +1 | x \in B^* ] \le \frac 1 2 - \frac \alpha 2$. Thus, $B_{best}$ will also satisfy $\Pr_S[ y = +1 | x \in B_{best} ] \le \frac 1 2 - \frac \alpha 2$.
By the closeness guarantee of the empirical distribution, we get that $\Pr[ y = +1 | x \in B_{best} ] \le \frac 1 2 - \frac \alpha 4$ and $\Pr[ x \in B_{best} ] > \frac {\alpha} {9 (2d)^k}$.

We now bound the error of the hypothesis $$h(x) = 
				\begin{cases} 
				-1 & x \in B_{best}\\
				z & \text{otherwise}
				\end{cases}$$
Within the region $B_{best}$, it achieves error at most $\frac 1 2 - \frac \alpha 4$, while outside of $B_{best}$, the error is at most.
$\frac 1 2 + \frac \alpha {O(d)^{k}}$. Thus, the total error is at most
$ \frac 1 2 - \frac {\alpha^2} {O(d)^{k}}$ given that $\Pr[ x \in B_{best} ] > \frac {\alpha} {9 (2d)^k}$.

The main computational step of the algorithm is searching over all rectangles with $k$ inequalities. It suffices to only consider rectangles with samples as end points, thus the total runtime of the weak-learner is 
$O( dN )^k = \frac {k^k O(d)^{k^2+1}} {\alpha^{2 k}}$ 
as for every inequality there are $2d$ choices for the direction $v$ and $N$ choices for the threshold $t$.

\end{proof}

\subsection{Putting Everything Together}\label{ssec:app-thm}

Lemma~\ref{lem:weak-learner} shows that  algorithm $\wklbox$ is a $(\alpha, \frac {\alpha^2} {O(d)^k})$-Weak Learner for unions of $k$ high-dimensional rectangles in $d$ dimensions. Combined with Theorem~\ref{thm:boosting} we get that:

\begin{theorem} \label{thm:rect-formal}
There exists an algorithm that learns unions of $k$ rectangles in $d$ dimensions with Massart noise bounded by $\eta$, achieving misclassification error $\eta + \epsilon$ for $\epsilon > 0$. 
The total number of samples is $\frac {k d^{O(k)}} {\eta^2 \epsilon^8}$ and the total running time is
$ \frac{1}{\eta^3} \left(\frac{k d^{k}}{\epsilon} \right)^{k + O(1)}$.
\end{theorem}

\begin{proof}
Follows by a direct application of the weak learner to Theorem~\ref{thm:boosting} for $\alpha = \epsilon / 8$ and $\gamma = \frac {\epsilon^2} {O(d)^{k}} $.
\end{proof}

\newpage

\bibliographystyle{alpha}
\bibliography{allrefs}

\end{document}